%% file: main.tex
\documentclass{article}
\usepackage{iclr2027_conference,times}
\iclrfinalcopy
\usepackage[T1]{fontenc}
\usepackage{amsmath,amssymb,amsthm,mathtools}
\usepackage{graphicx,booktabs,microtype}
\usepackage{wrapfig,needspace}
\usepackage{tikz}
\usetikzlibrary{arrows.meta,positioning}
\usepackage{hyperref}
\usepackage{url}
\hypersetup{colorlinks=true,citecolor=blue,linkcolor=blue,urlcolor=blue,pdftitle={On the Limits of Maximal Coding Rate Reduction for Out-of-Distribution Generalisation},pdfauthor={Menghui Zhou, Gaoshan Bi, Vitaveska Lanfranchi, Po Yang}}
\newtheorem{theorem}{Theorem}[section]
\newtheorem{proposition}[theorem]{Proposition}
\newtheorem{lemma}[theorem]{Lemma}

\theoremstyle{definition}

\newcommand{\mcr}{\ensuremath{\mathrm{MCR}^{2}}}
\newcommand{\icr}{\ensuremath{\mathrm{ICR}^{2}}}
\newcommand{\E}{\mathbb{E}}
\newcommand{\R}{\mathbb{R}}
\newcommand{\cE}{\mathcal{E}_{\mathrm{tr}}}
\newcommand{\cF}{\mathcal{F}}
\newcommand{\cX}{\mathcal{X}}
\newcommand{\cD}{\mathcal{D}}
\newcommand{\sphere}{\mathbb{S}^{1}}
\newcommand{\tr}{\operatorname{tr}}
\newcommand{\diag}{\operatorname{diag}}
\newcommand{\Risk}{\mathcal{L}}
\newcommand{\Jstar}{J_{\star}(a)}
\newcommand{\gap}{\Gamma_a}
\DeclareMathOperator*{\argmax}{arg\,max}

\title{On the Limits of Maximal Coding Rate Reduction for Out-of-Distribution Generalisation}
\author{{\bfseries Menghui Zhou \quad Gaoshan Bi \quad Vitaveska Lanfranchi \quad Po Yang} \\
{\normalfont School of Computer Science} \\
{\normalfont University of Sheffield} \\
{\footnotesize\ttfamily \{menghui.zhou, gaoshan.bi, v.lanfranchi, po.yang\}@sheffield.ac.uk}}
\begin{document}
\maketitle
\lhead{arXiv preprint}

\input{sections/abstract}
\input{sections/introduction}

\input{sections/setup}

\input{sections/failure}

\input{sections/invariance}
\input{sections/discussion}
\input{sections/statements}
\bibliography{references}
\bibliographystyle{iclr2027_conference}
\clearpage
\appendix
\input{sections/appendix_overview}
\input{sections/related}
\clearpage
\input{sections/proofs}
\input{sections/secondary}

\input{sections/experiments}
\input{sections/waterbirds_setup}
\end{document}

%% file: sections/abstract.tex
\begin{abstract}
Substantial efforts have been devoted to making deep learning objectives, representations, and architectures interpretable, with the goal of improving the safety, robustness, and generalisation of learning systems in diverse real-world applications. The recently proposed maximal coding rate reduction (\mcr) offers a promising information-theoretic framework for learning structured, discriminative representations of class-wise submanifolds and has inspired interpretable white-box architectures. However, we observe that \mcr{} can completely fail under distribution shift, motivating our study of its out-of-distribution (OOD) generalisation limits. We establish two limitations of \mcr{} for OOD generalisation. First, the \mcr{} objective alone can admit complete prediction failure: a representation based entirely on unstable environmental features can achieve the global coding optimum yet fail completely after correlation reversal, despite an available perfectly stable feature. This exact-optimum example includes test inputs that cannot occur during training. Even when every possible test input can also occur during training, coding quality can be arbitrarily close to optimal while prediction error is arbitrarily close to 100\%. Second, directly incorporating the invariance principle underlying widely successful invariant risk minimisation (IRM) and risk extrapolation (REx) does not eliminate this failure. The failing representation admits the same optimal coding operator across training environments, showing that shared coding optimality does not ensure stable prediction. Reliable OOD guarantees for \mcr{} therefore require additional new assumptions or learning principles that establish stable predictive relationships across environments.
\end{abstract}

%% file: sections/introduction.tex
\section{Introduction}
\label{sec:intro}

Over the past decade, deep learning has transformed a wide range of applications, from language processing~\citep{devlin2019bert} and computer vision~\citep{he2016deep} to scientific discovery~\citep{abramson2024accurate} and robotics~\citep{kaufmann2023champion}. Despite these advances, modern deep networks are still designed largely through trial and error~\citep{elsken2019survey, zhou2024integrating}. Researchers often choose architectures and training strategies through extensive experimentation~\citep{zoph2017neural,hutter2019automated}. Yet it remains difficult to explain the resulting systems through a clear set of computational principles~\citep{lipton2018mythos}. This gap motivates the development of interpretable learning systems~\citep{rudin2019stop} whose objectives, representations, and computations can be understood within a common mathematical framework~\citep{ma2022principles}.

Maximal coding rate reduction (\mcr) offers a promising approach to this goal~\citep{yu2020learning}. It starts with a basic question about representation learning: \emph{what structure should a network learn and preserve from data?} The framework assumes that high-dimensional data from each class often lie near a low-dimensional submanifold~\citep{ma2022principles}. A useful representation should capture this class-specific structure, preserve meaningful variation within each class, and separate different classes. Building on earlier work that connects lossy data coding with data segmentation~\citep{ma2007segmentation}, \mcr{} introduces an information-theoretic objective that maximises the difference between the coding rate of the complete dataset and the average coding rate of its individual classes. This objective gives a precise way to describe representation geometry: it encourages compact class representations, diversity within each class, and separation between classes~\citep{yu2020learning,chan2022redunet,wang2024geometry, yu2024white}.

The importance of \mcr{} extends from representation geometry to principled network design. Its geometric analyses describe solutions in which classes occupy mutually orthogonal subspaces and retain variation within each subspace~\citep{yu2020learning,wang2024geometry}. This connects a computable objective value to properties of the learned representation. Crucially, coding-rate principles also provide a way to derive white-box architectures whose layers have explicit mathematical roles. ReduNet unfolds the optimisation of \mcr{} into network layers with optimisation, statistical, and geometric interpretations~\citep{chan2022redunet}. Building on the coding-rate perspective, CRATE derives Transformer-like architectures from a sparse rate-reduction objective, giving attention and subsequent layers explicit roles in compressing and sparsifying representations~\citep{yu2024white}. These developments connect learning objectives, representation structure, and network computation within a common framework, making coding-rate principles relevant to the broader goal of interpretable deep learning. Coding-rate principles have also been extended to incremental learning~\citep{wu2021incremental,tong2023incremental}, joint discriminative and generative learning~\citep{dai2022ctrl,tong2024unsupervised}, image clustering~\citep{chu2024image}, and self-supervised representation learning~\citep{wu2025simplifying}. Their influence on both representation learning and architecture design makes it important to understand what the coding objective itself can guarantee about prediction reliability.

Despite these advantages, we find in practice that representations learned with \mcr{} can lead to severe out-of-distribution (OOD) prediction failures. On the Waterbirds dataset~\citep{sagawa2020distributionally}, an encoder trained with \mcr{} performs poorly on test groups whose background--label associations oppose those prevalent during training (Table~\ref{tab:waterbirds-test}). The saliency maps in Figure~\ref{fig:MCR2-spurious} further illustrate its sensitivity to background regions. These failures raise concerns about the reliability of \mcr{} representations in real-world deployment and provide the main motivation for examining what the objective can guarantee under distribution shift.

We first analyse the OOD generalisation of \mcr{} and show that optimal coding geometry alone does not guarantee reliable prediction in a new environment. We use a simple two-class model to isolate the failure mechanism and demonstrate that this limitation arises even in a basic classification setting. An encoder using only an environmental feature can achieve the global coding optimum, despite the availability of a perfectly reliable feature. When the environmental feature's relationship with the label reverses, the classes remain well separated, but the source-optimal classifier predicts every test label incorrectly. This noiseless example includes test inputs that cannot occur during training. We further show that even when every possible test input can also occur during training, coding quality can approach the global optimum while prediction error approaches 100\%.

Given the success of the invariance principle in OOD learning approaches such as invariant risk minimisation (IRM; \citealp{arjovsky2019invariant}) and risk extrapolation (REx; \citealp{krueger2021out}), a natural next step is to examine whether this principle can overcome the limitation of \mcr{}. We show that even directly incorporating this principle into the coding-rate framework does not eliminate OOD prediction failure. In this formulation, invariance requires the same coding operator to be optimal in every training environment. Yet a representation relying entirely on an unstable environmental feature satisfies this requirement and still fails when the feature's relationship with labels reverses.

To the best of our knowledge, this work provides the first systematic analysis of the limitations of influential \mcr{} framework for OOD generalisation.  Our main contributions are as follows:
\begin{itemize}
	\item \textbf{The \mcr{} objective alone does not guarantee reliable OOD generalisation.}
We first prove that \mcr{} can fail completely in OOD generalisation even in a simple noiseless two-class model. An encoder relying solely on an environmental feature achieves the global \mcr{} optimum, despite the availability of a perfectly stable predictive feature. When the envirnmental correlation reverses, its source-optimal classifier misclassifies every test example. This construction includes test inputs that cannot occur during training.
	
	We then extend the analysis to noisy environments where training and test data share the same possible inputs. By deriving the exact coding-rate reduction, prediction risk, and global coding optimum, we show that an arbitrarily small gap from the \mcr{} optimum can coexist with almost complete prediction failure, for any fixed coding precision and any finite number of training environments.
	
\item \textbf{Directly incorporating the invariance principle into \mcr{} does not eliminate OOD failure.}
Motivated by the success of invariance-based approaches such as IRM and REx, we directly incorporate this principle into the \mcr{} framework through invariant coding rate reduction (\icr), which requires the same coding operator to be optimal across training environments. We prove that an encoder relying solely on an unstable environmental feature satisfies this requirement exactly, yet its source-optimal classifier can still misclassify almost every test example after correlation reversal. This failure persists even when training and test data share the same possible inputs and the coding value is arbitrarily close to the constrained global optimum.
\end{itemize}

These results distinguish the ability to organise class representations from the ability to predict reliably across environments. For representations learned with \mcr{}, these results identify the need for new assumptions or learning principles that address stable predictive relationships beyond the coding objective and the shared-optimality requirement studied here.

%% file: sections/setup.tex
\section{Method}
\label{sec:setup}
We first review the \mcr{} objective and use an object-and-background example and existing Waterbirds results to explain the motivation for our analysis. We then define the population objective, prediction risk, and two-class model. Vectors and matrices are written in bold throughout the paper.

\subsection{The maximal coding rate reduction objective}

\begin{wrapfigure}{r}{0.43\textwidth}
\vspace{-6pt}
\centering
\includegraphics[width=\linewidth]{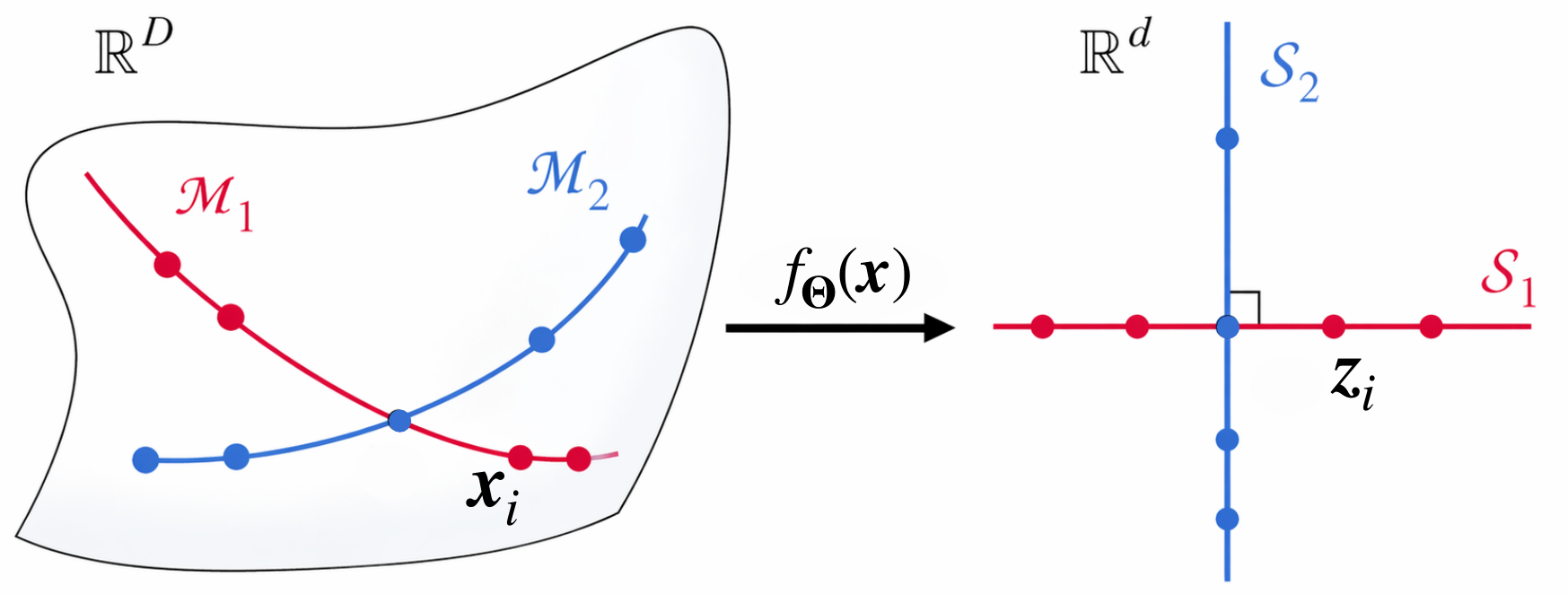}
\caption{Class-wise submanifolds mapped to orthogonal subspaces by \mcr{}. Adapted from Figure~1 of \citet{yu2020learning}.}
\label{fig:MCR2}
\vspace{-6pt}
\end{wrapfigure}

Consider a dataset $\mathbf{X}=[\mathbf{x}_1,\ldots,\mathbf{x}_n]\in\R^{D\times n}$ with $n$ samples from $K$ classes. Representation learning maps these high-dimensional observations to lower-dimensional features that retain useful information for tasks such as classification and generation. The \emph{manifold hypothesis} suggests that samples from each class concentrate near a low-dimensional submanifold $\mathcal{M}_y\subset\R^D$, so the full dataset lies near their union $\mathcal{M}=\bigcup_{y=1}^K\mathcal{M}_y$~\citep{hein2005intrinsic,pope2021intrinsic,wright2022high,ma2022principles}. Building on this hypothesis, the \mcr{} learning objective aims to capture the structure of each class submanifold and organise the resulting representations into mutually orthogonal low-dimensional subspaces~\citep{yu2020learning}. To learn such representations, a nonlinear encoder $f_{\boldsymbol{\Theta}}:\R^D\to\R^d$ maps each input $\mathbf{x}_i$ to a feature vector $\mathbf{z}_i=f_{\boldsymbol{\Theta}}(\mathbf{x}_i)$. Collecting these features gives the representation matrix $\mathbf{Z}=[\mathbf{z}_1,\ldots,\mathbf{z}_n]\in\R^{d\times n}$, typically with $d\ll D$.

In the supervised setting, let $\boldsymbol{\Pi}_y\in\R^{n\times n}$ be a diagonal membership matrix whose $i$th diagonal entry is one if sample $i$ belongs to class $y$, and zero otherwise. The number of samples in class $y$ is $n_y=\tr(\boldsymbol{\Pi}_y)>0$. The  learning objective of \mcr{} is  maximising the  coding-rate reduction:
\begin{equation}
\begin{aligned}
\widehat J(\mathbf{Z})
&=\underbrace{\frac12\log\det\!\left(\mathbf{I}_d+\tfrac{d}{n\epsilon^2}\mathbf{Z}\mathbf{Z}^\top\right)}_{\text{Expansion: }\mathcal R(\mathbf{Z})}
-\underbrace{\sum_{y=1}^K\tfrac{n_y}{2n}\log\det\!\left(\mathbf{I}_d+\tfrac{d}{n_y\epsilon^2}\mathbf{Z}\boldsymbol{\Pi}_y\mathbf{Z}^\top\right)}_{\text{Compression: }\mathcal R_c(\mathbf{Z},\boldsymbol{\Pi})},\\
&\text{subject to}\quad \|\mathbf{z}_i\|_2 = 1,
\qquad i=1,\ldots,n.
\end{aligned}
\label{eq:sample}
\end{equation}
where $\epsilon>0$ is the coding precision. The normalisation constraint ensures a common feature scale, allowing fair comparisons of coding-rate reduction across representations.

The expansion term $\mathcal R(\mathbf{Z})$ encourages the features to spread out in the feature space. The compression term $\mathcal R_c(\mathbf{Z},\boldsymbol{\Pi})$ encourages the features of each class to occupy a compact subspace. Together, this learning objective encourages the features of each class to lie in a low-dimensional linear subspace, with mutually orthogonal subspaces across classes \citep{yu2020learning}, as illustrated in Figure~\ref{fig:MCR2}. At an orthogonal solution,
\begin{equation}
	(\mathbf{Z}\boldsymbol{\Pi}_y)^\top
	(\mathbf{Z}\boldsymbol{\Pi}_{y'})=\mathbf{0},
	\qquad y\neq y'.
	\label{eq:class-orthogonality}
\end{equation}

\subsection{Concerns about learning environmental features}
Although \mcr{} encourages structured class representations, the objective does not specify which input factors the encoder should use. Both object features and background features can have low-dimensional structure and can help separate the training classes.

\begin{wrapfigure}{r}{0.35\textwidth}
	\vspace{-6pt}
	\centering
	\includegraphics[width=\linewidth]{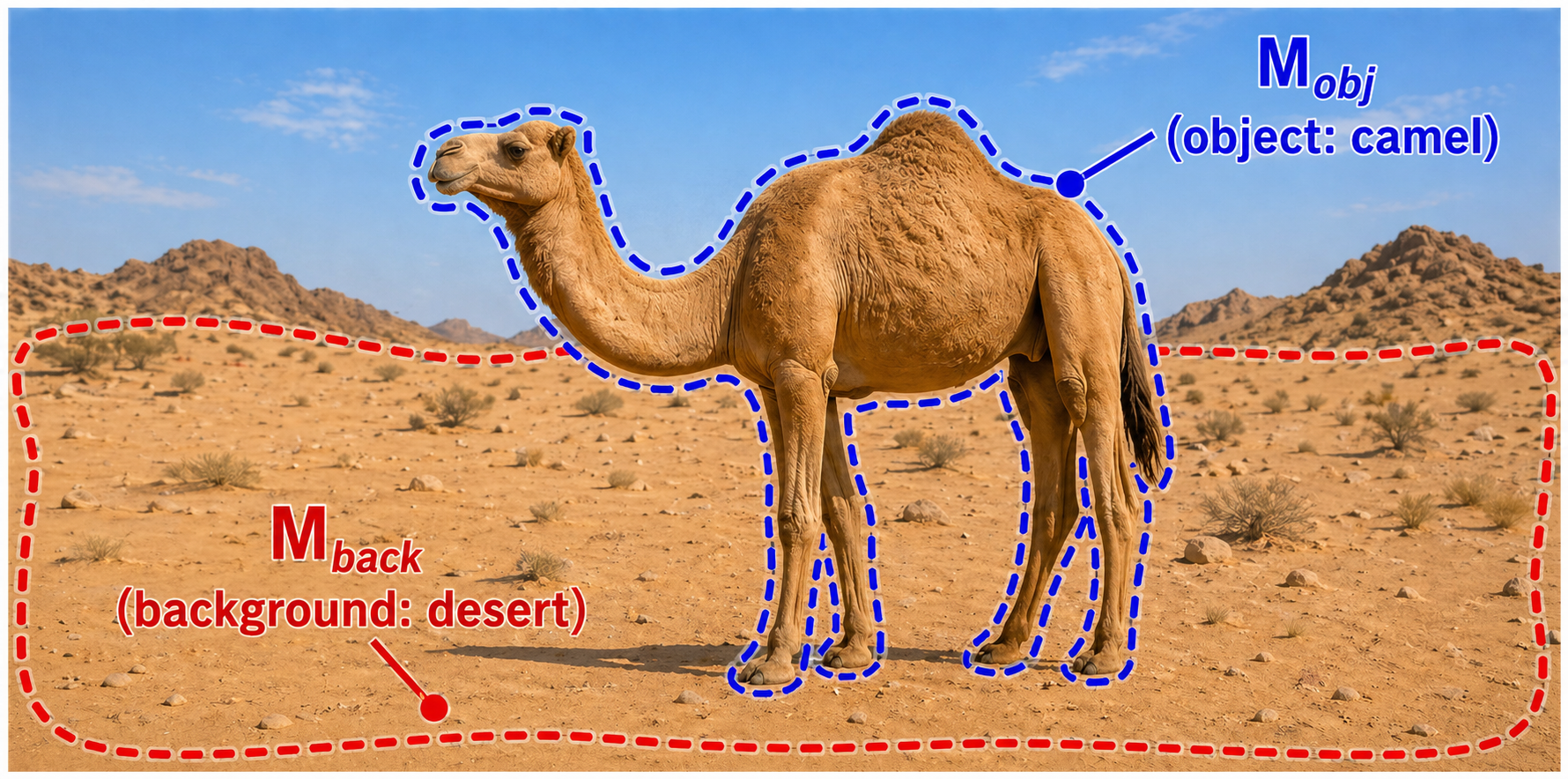}
	\caption{Object and background structures in a camel image.}
	\label{fig:object_vs_background}
		\vspace{-12pt}
\end{wrapfigure}
Our intuition comes from the following example. Consider the camel image in Figure~\ref{fig:object_vs_background}. Both the object and the background lie on low-dimensional submanifolds, denoted by $\mathcal M_{\mathrm{obj}}$ and $\mathcal M_{\mathrm{back}}$, respectively.  If capturing $\mathcal M_{\mathrm{back}}$ yields a higher coding-rate reduction than capturing $\mathcal M_{\mathrm{obj}}$, the \mcr{} objective favours the background representation. The model can therefore achieve strong class separation during training by learning background cues. When the background--label relationship changes, however, a classifier relying on these cues may fail.

\paragraph{Evidence from Waterbirds.}
\begin{wraptable}{r}{0.32\textwidth}
	\vspace{-6pt}
	\centering
	\caption{Accuracy (\%) on the two background-mismatched Waterbirds test groups.}
	\label{tab:waterbirds-test}
	\resizebox{\linewidth}{!}{%
		\begin{tabular}{ccc}
			\toprule
			Overall & $(y=0,p=1)$ & $(y=1,p=0)$\\
			\midrule
			46.84 & 57.07 & 10.90\\
			\bottomrule
	\end{tabular}}
	\vspace{-6pt}
\end{wraptable}
To examine whether this concern arises in practice, we train a ResNet-18 encoder~\citep{he2016deep} from scratch using \mcr{} on Waterbirds~\citep{sagawa2020distributionally}. The task is to distinguish landbirds ($y=0$) from waterbirds ($y=1$). Each image also has a land ($p=0$) or water ($p=1$) background. In the training set, $95\%$ of waterbirds appear on water backgrounds and $95\%$ of landbirds appear on land backgrounds; the remaining $5\%$ of each class have the opposite background. Thus, the background is strongly associated with the bird label during training. The experimental setup is described in Appendix~\ref{app:waterbirds-setup}.

We evaluate the encoder on landbirds with water backgrounds ($y=0,p=1$; 2,255 images) and waterbirds with land backgrounds ($y=1,p=0$; 642 images), reversing the predominant training background--label association. Table~\ref{tab:waterbirds-test} reports per-group and combined accuracy.
\begin{wrapfigure}{r}{0.35\textwidth}
	\vspace{-5pt}
	\centering
	\includegraphics[width=0.48\linewidth]{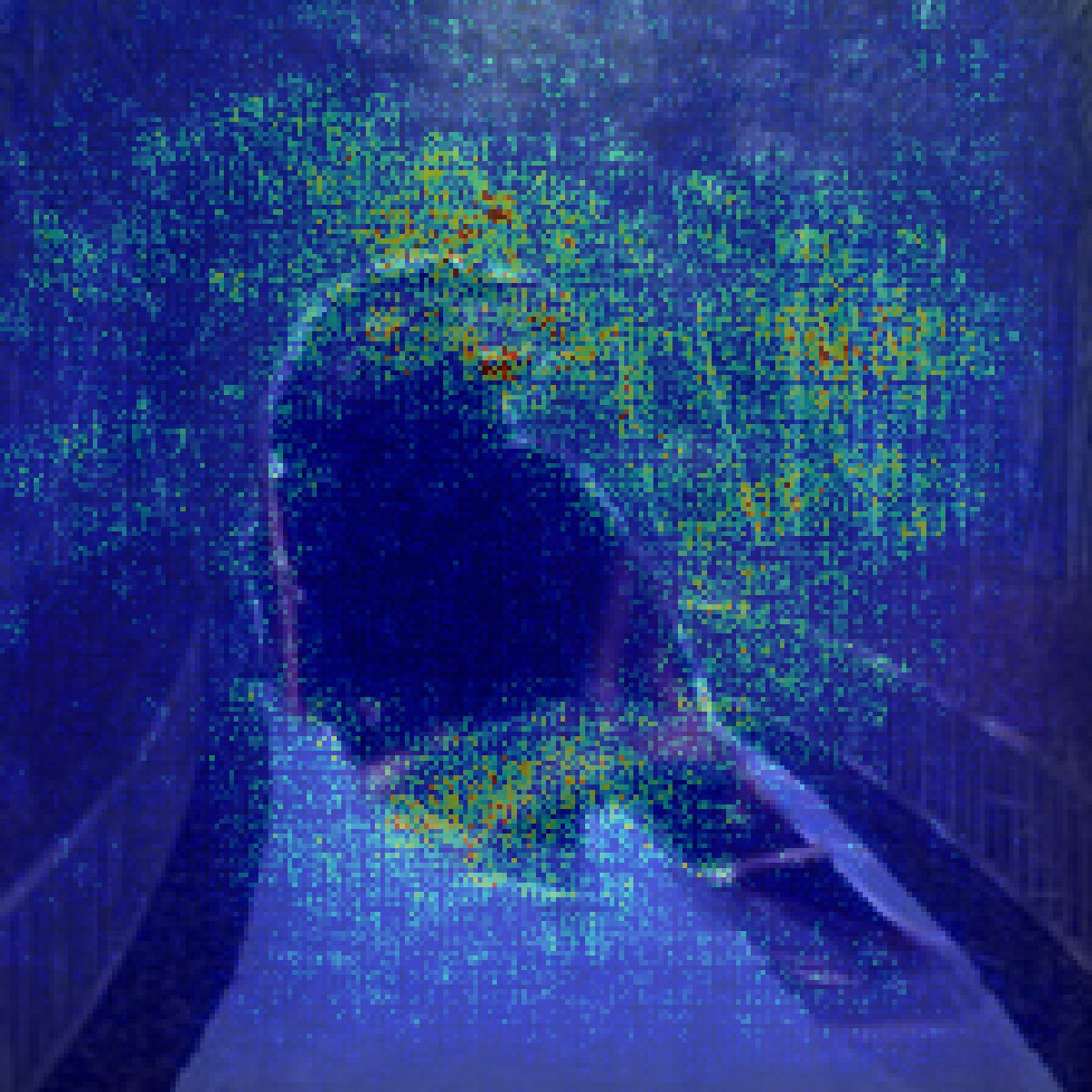}\hfill
	\includegraphics[width=0.48\linewidth]{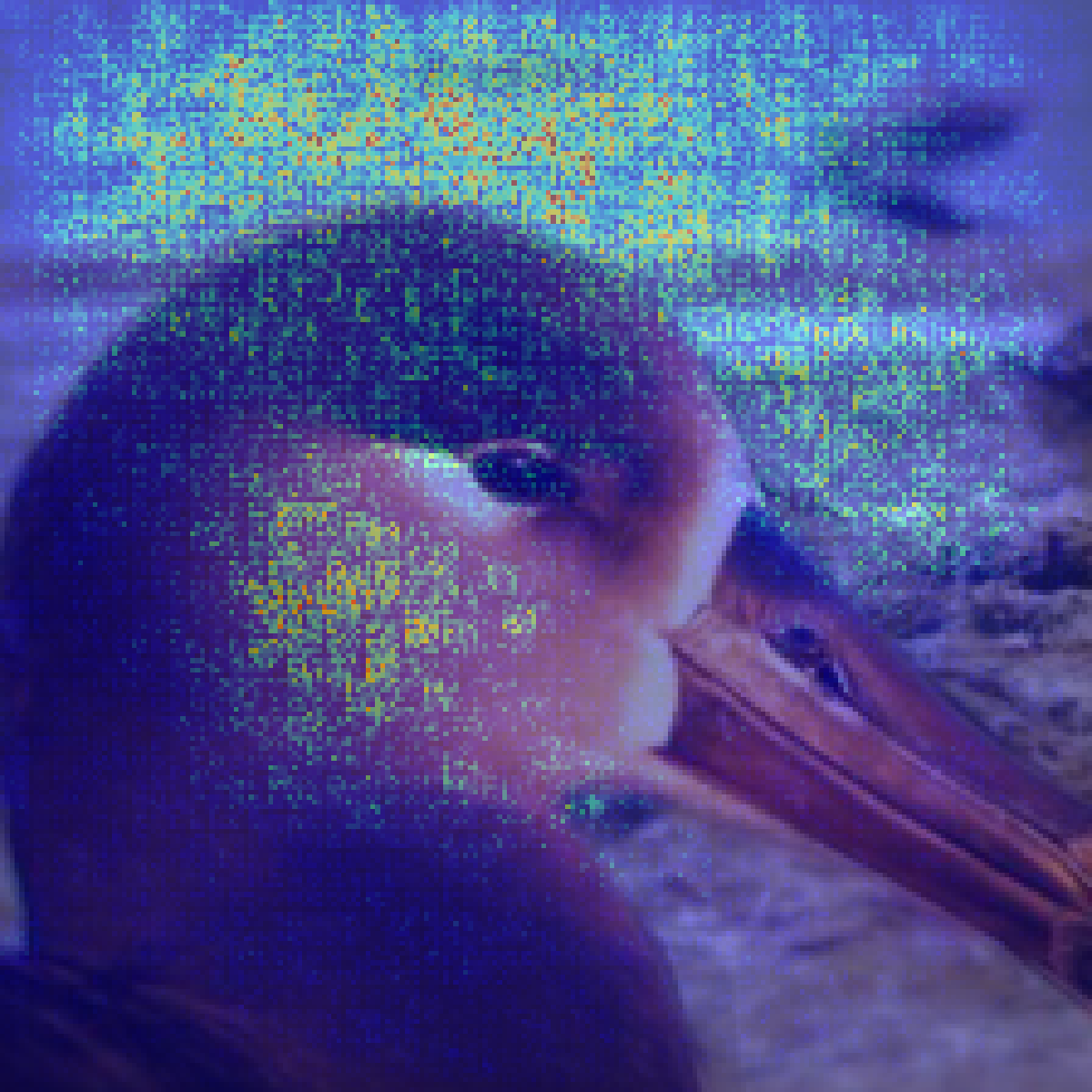}
	\caption{\mcr{} saliency maps on Waterbirds.}
	\label{fig:MCR2-spurious}
	\vspace{-5pt}
\end{wrapfigure}
 Combined accuracy is $46.84\%$, below the $50\%$ expected from uniform random guessing, while accuracy on waterbirds with land backgrounds is only $10.90\%$, indicating severe failure under this shift.
Saliency maps for one training image per class show stronger sensitivity to background regions than to much of the bird (Figure~\ref{fig:MCR2-spurious}). Together with the test results, they suggest reliance on background cues, illustrating that structured representations learned by \mcr{} alone do not ensure stable prediction across environments.

These experimental observations motivate a theoretical investigation of the limitations of \mcr{} under distribution shift. To study these limitations beyond the specific Waterbirds setting, we analyse a model with a perfectly stable feature and an environmental feature, allowing us to calculate the coding objective and prediction risk exactly and examine whether coding quality guarantees OOD prediction reliability.

\subsection{Population formulation and prediction risk}
In this section, we first express the \mcr{} objective in terms of data distributions and define prediction risk, allowing us to compare the objective value with classification performance across environments. We then introduce a model containing both a stable feature and an environmental feature whose association with the label can change. This model lets us study whether the objective distinguishes representations that support stable prediction from those that rely on environmental associations.

\paragraph{Population \mcr{} objective.}
Let $P_e$ denote the distribution of $(\mathbf{X},Y)$ in environment $e$, with class probabilities $p_y^e=P_e(Y=y)$. For an encoder $f:\cX\to\mathbb S^{d-1}$, define the uncentred second moments
\begin{equation}
	\mathbf{M}_{e,y}(f)=\E_e[f(\mathbf{X})f(\mathbf{X})^\top\mid Y=y],
	\qquad
	\mathbf{M}_e(f)=\sum_y p_y^e\mathbf{M}_{e,y}(f).
	\label{eq:moments}
\end{equation}
Replacing the empirical moments in Equation~\eqref{eq:sample} with these expectations gives the population \mcr{} objective:
\begin{equation}
	J_e(f)=\frac12\log\det(\mathbf{I}_d+a\mathbf{M}_e(f))
	-\frac12\sum_y p_y^e\log\det(\mathbf{I}_d+a\mathbf{M}_{e,y}(f)),
	\qquad a=\frac{d}{\epsilon^2}.
	\label{eq:population}
\end{equation}
Concavity of $\log\det$ ensures $J_e(f)\geq0$.

\paragraph{Learning and evaluation.}
For a finite set $\cE$ of training environments with weights $w_e>0$ satisfying $\sum_{e\in\cE}w_e=1$, define
\begin{equation}
	J_{\mathrm{tr}}(f)=\sum_{e\in\cE}w_eJ_e(f),
	\qquad
	\Risk_e(g\circ f)=P_e(g(f(\mathbf{X}))\neq Y).
	\label{eq:trainrisk}
\end{equation}
Here, $J_{\mathrm{tr}}$ is the training \mcr{} objective and $\Risk_e$ is the classification error in environment $e$. For each representation, we use a source-optimal classifier that minimises $\sum_{e\in\cE}w_e\Risk_e(g\circ f)$ over all classifiers and remains fixed during testing. This isolates representation quality from classifier estimation or optimisation errors. For both encoders below, a linear decision rule attains this optimum.

\paragraph{A stable factor and an environmental factor.}
Our main results consider $d=2$, binary labels, and inputs $\mathbf{X}=(C,S)\in\{0,1\}^2$, where
\begin{equation}
	Y\sim\operatorname{Bernoulli}(1/2),\qquad
	C=Y,\qquad S=Y\oplus B_e,\qquad
	B_e\sim\operatorname{Bernoulli}(\delta_e),\quad B_e\perp Y.
	\label{eq:model}
\end{equation}
Here, $\oplus$ denotes exclusive-or. The stable feature $C$ predicts the label perfectly in every environment, whereas the environmental feature $S$ disagrees with it with probability $\delta_e$. These are predictive relationships and require no causal assumptions.

We consider all unit-norm encoders
\begin{equation}
	\cF=\{f:\{0,1\}^2\to\sphere\}.
	\label{eq:class}
\end{equation}
With $\mathbf{b}_0=(1,0)^\top$ and $\mathbf{b}_1=(0,1)^\top$, define
\begin{equation}
	f_C(c,s)=\mathbf{b}_c,\qquad
	f_S(c,s)=\mathbf{b}_s,\qquad
	g_0(\mathbf z)=\mathbf1\{z_2>z_1\}.
	\label{eq:encoders}
\end{equation}
The encoders $f_C$ and $f_S$ use the stable and environmental features, respectively, mapping each to the same two orthogonal unit vectors. Their input probabilities are
\begin{equation}
	P_e(C=c,S=s)=\tfrac12\big[(1-\delta_e)\mathbf1\{c=s\}
	+\delta_e\mathbf1\{c\neq s\}\big].
	\label{eq:support}
\end{equation}
For $0<\delta_e<1$, all four inputs have positive probability. Varying $\delta_e$ therefore changes the environmental association while preserving input support, allowing us to compare the two encoders' \mcr{} objective values and prediction risks under this shift.

%% file: sections/failure.tex
\section{Discriminative coding geometry does not ensure OOD generalisation}
\label{sec:failure}
We show how a representation can have a coding value close to the best possible value while a classifier trained on it predicts almost every target label incorrectly. We first identify the global coding optimum and then compare it with the environmental encoder. In  this section $a=2/\epsilon^2>0$ is fixed.

\begin{lemma}[Sharp binary coding bound]
\label{lem:bound}
For any balanced binary distribution and any unit-norm encoder into $\R^2$,
\begin{equation}
J_e(f)\leq \Jstar:=\log(1+a/2)-\tfrac12\log(1+a).
\label{eq:optimum}
\end{equation}
Equality holds if and only if the two class-conditional second moments are rank-one projectors onto orthogonal directions. In model~\eqref{eq:model}, $f_C$ attains equality in every environment, so $$\max_{f\in\cF}J_{\mathrm{tr}}(f)=\Jstar.$$
\end{lemma}

Lemma~\ref{lem:bound} gives a precise meaning to successful coding geometry in this model: each class occupies one direction, and the two directions are orthogonal. The stable encoder achieves this structure in every environment, so optimal coding and reliable prediction are both possible. The lemma also provides a global benchmark over the entire encoder class. The next proposition measures how closely a representation based only on the environmental feature can approach it.

\begin{proposition}[Exact environmental objective and prediction risk]
\label{prop:exact}
Under Equation~\eqref{eq:model}, define
\begin{equation}
\gap(\delta):=\frac12\log\!\left(1+\frac{a^2}{1+a}\delta(1-\delta)\right).
\label{eq:gap}
\end{equation}
Then $J_e(f_S)=\Jstar-\gap(\delta_e)$ and $\Risk_e(g_0\circ f_S)=\delta_e$. If $\bar\delta:=\sum_e w_e\delta_e<1/2$, $g_0$ is the unique source-optimal decision rule on the two output codes. In contrast, $\Risk_e(g_0\circ f_C)=0$ for every $e$.
\end{proposition}

When $\delta_e$ is small, the environmental feature usually agrees with the label. It therefore produces both a small coding gap and a small source prediction error. Yet the classifier has learned to read that feature, whose reliability can change across environments. Proposition~\ref{prop:exact} makes the consequence explicit: the classifier's error in any environment is exactly the probability that the environmental feature disagrees with the label. The next result shows how large the gap between coding quality and target reliability can become. 

\begin{theorem}[Arbitrarily small coding gap and almost maximal target error]
\label{thm:near}
Fix $a>0$, a positive integer $E$, positive source weights summing to one, and a tolerance $\eta>0$. There exist $E$ training environments and a target environment from Equation~\eqref{eq:model}, all with the same input support, such that
\begin{equation}
0\leq \max_{f\in\cF}J_{\mathrm{tr}}(f)-J_{\mathrm{tr}}(f_S)\leq\eta,
\qquad
\Risk_{\mathrm{te}}(g_0\circ f_S)\geq1-\eta,
\label{eq:near-failure}
\end{equation}
while $g_0$ is source-optimal on $f_S$ and $\Risk_{\mathrm{te}}(g_0\circ f_C)=0$. The source mismatch probabilities can be chosen distinct. More explicitly, if $0<\delta_e\leq\delta_{\max}<1/2$, then
\begin{equation}
\Jstar-J_{\mathrm{tr}}(f_S)
=\sum_e w_e\gap(\delta_e)
\leq \frac{a^2}{2(1+a)}\delta_{\max}.
\label{eq:gapbound}
\end{equation}
Choosing $\delta_{\mathrm{te}}=1-\delta_{\max}$ yields target error $1-\delta_{\max}$.
\end{theorem}

\paragraph{Meaning and scope.}
Theorem~\ref{thm:near} shows that an arbitrarily small coding gap can coexist with an almost maximal target error. Every possible target input already has positive source probability. What changes is the frequency of inputs on which the environmental feature predicts the wrong label: these inputs are rare during training and common at test time. Observing the same possible inputs therefore does not by itself ensure reliable prediction when their probabilities change.

The distinction between exact and approximate optimisation matters here. At positive noise, $f_C$ has a strictly higher objective than $f_S$. More strongly, in this finite model with full source support, every exact global optimiser yields zero target error with an unrestricted source-optimal classifier (Appendix~\ref{app:exact-stability}). The theorem establishes that an arbitrarily small objective gap does not provide a uniform OOD guarantee over the permitted shifts. It does not assert failure of exact optimisers in this setting or exclude guarantees that additionally control the size of the shift.

At the noiseless boundary, one obtains an exact, but support-changing, statement. If every source has $\delta_e=0$, both $f_C$ and $f_S$ globally maximise the source objective. At a target with $\delta_{\mathrm{te}}=1$, the source-optimal classifier on $f_S$ has error one. This example establishes complete prediction failure at an exact optimum with changing input support, whereas Theorem~\ref{thm:near} establishes near-optimality and arbitrarily severe failure under identical input support.

\subsection{Coding quality can survive correlation reversal}
\label{sec:reversal}
Note that the failure is not necessarily accompanied by a deterioration of coding geometry on the target distribution. From Equation~\eqref{eq:gap},
\begin{equation}
J_{\delta}(f_S)=J_{1-\delta}(f_S),\qquad
\Risk_{\delta}(g_0\circ f_S)=\delta,\qquad
\Risk_{1-\delta}(g_0\circ f_S)=1-\delta.
\label{eq:reversal}
\end{equation}

For example, if the environmental cue agrees with the label in $99\%$ of source samples, the source classifier has error $1\%$. After reversal, its error becomes $99\%$, while the coding value remains exactly the same. The representation still distinguishes the classes, but the source classifier uses the wrong label assignment for most target inputs.

Both environments have the same marginal representation law, namely the uniform distribution on $\{\mathbf{b}_0,\mathbf{b}_1\}$. Reversal swaps the two class-conditional moment matrices, leaving their average and their weighted log-determinants unchanged. The target remains easy to classify \emph{if a new target head is fitted}; its Bayes head on $f_S$ assigns the opposite labels. The task here is to use the source head without such refitting.

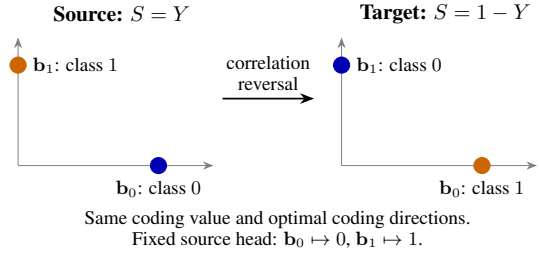
\begin{wrapfigure}{r}{0.5\textwidth}
	\centering
	\resizebox{\linewidth}{!}{%
		\begin{tikzpicture}[font=\small,>=Stealth]
			\node[font=\bfseries] at (1.7,2.5) {Source: $S=Y$};
			\node[font=\bfseries] at (7.0,2.5) {Target: $S=1-Y$};
			\draw[->,gray] (0,0)--(3.2,0);
			\draw[->,gray] (0,0)--(0,2.1);
			\fill[blue!70!black] (2.3,0) circle (4pt);
			\fill[orange!80!black] (0,1.65) circle (4pt);
			\node[below,align=center] at (2.3,-0.13) {$\mathbf{b}_0$: class $0$};
			\node[right,align=left] at (0.12,1.65) {$\mathbf{b}_1$: class $1$};
			\draw[->,gray] (5.3,0)--(8.5,0);
			\draw[->,gray] (5.3,0)--(5.3,2.1);
			\fill[orange!80!black] (7.6,0) circle (4pt);
			\fill[blue!70!black] (5.3,1.65) circle (4pt);
			\node[below,align=center] at (7.6,-0.13) {$\mathbf{b}_0$: class $1$};
			\node[right,align=left] at (5.42,1.65) {$\mathbf{b}_1$: class $0$};
			\draw[->,thick] (3.35,1.1)--node[above,align=center,font=\footnotesize]{correlation\\reversal}(4.85,1.1);
			\node[align=center] at (4.25,-1.05) {Same coding value and optimal coding directions.\\Fixed source head: $\mathbf{b}_0\mapsto0$, $\mathbf{b}_1\mapsto1$.};
	\end{tikzpicture}}
	\caption{Optimal coding geometry survives correlation reversal; prediction fails.}
	\label{fig:mechanism}
	\vspace{-10pt}
\end{wrapfigure}

Figure~\ref{fig:mechanism} illustrates the noiseless limit. In the positive-noise case, each class occupies both codes, with a large mass on one and a small mass on the other. Thus exact class-subspace orthogonality holds at the limit, whereas the identical-support theorem establishes near-optimality with small cross-class contamination. We make no claim of exact orthogonality at positive noise. Also, target coding-rate reduction uses target labels in Equation~\eqref{eq:population}; it is a theoretical diagnostic, not an unlabelled deployment statistic.

%% file: sections/invariance.tex
\section{Shared optimal coding does not ensure generalisation}
\label{sec:invariance}
The preceding results show why good coding geometry alone does not ensure stable prediction. Given the success of invariance-based methods such as invariant risk minimisation (IRM)~\citep{arjovsky2019invariant} and risk extrapolation (REx)~\citep{krueger2021out} in OOD generalisation, it is natural to ask whether incorporating the invariance principle into \mcr{} can address the failures established above. We show, however, that a direct formulation based on shared optimal coding can still admit severe OOD prediction failure.

Specifically, IRM and REx impose invariance at the level of prediction: IRM seeks a classifier that is simultaneously optimal across training environments, while REx penalises differences in prediction risk across environments. Within the \mcr{} framework, the learning criterion evaluates representation geometry through coding-rate reduction. A direct analogue of shared classifier optimality therefore requires a common linear coding operator that transforms the representation to maximise the \mcr{} objective in every training environment. For convenience, we refer to this shared-optimality requirement as invariant coding rate reduction (\icr{}). This constraint enforces agreement on optimal coding structure across environments but, as our analysis shows, does not ensure stable predictive relationships.

\subsection{A shared-optimal-coding constraint}
The coding operator acts on the learned representation and changes the directions assigned to its outputs. If different environments favoured incompatible operators, a shared-optimality requirement could reject the representation. The issue is whether agreement on an operator also ensures that a classifier can keep using the same label assignments. For $f\in\cF$, define the set of admissible linear coding operators by
\begin{equation}
\cD(f)=\{\mathbf{D}\in\R^{2\times2}:\|\mathbf{D}f(\mathbf{x})\|_2=1\ \text{for all }\mathbf{x}\in\cX\}.
\label{eq:operator-domain}
\end{equation}
Here $\cX=\{0,1\}^2$ is fixed, and $\mathbf{I}_2\in\cD(f)$. The constraint applies to the encoded representations $\mathbf{D}f(\mathbf{x})$ rather than to the operator's Frobenius norm. For $f_C$ and $f_S$, which attain both basis codes, this set is exactly
\begin{equation}
\cD(f_C)=\cD(f_S)=\{[\mathbf{u},\mathbf{v}]:\mathbf{u},\mathbf{v}\in\sphere\}.
\label{eq:columns}
\end{equation}
It includes non-orthogonal columns and rank-one operators; it is not restricted to rotations for which rate reduction is automatically unchanged.

Consider the following idealised invariant coding-rate reduction objective:
\begin{equation}
\begin{aligned}
J^*_{\mathrm{ICR}}=\sup_{f\in\cF,\,\mathbf{D}\in\cD(f)}\quad
&\sum_{e\in\cE}w_eJ_e(\mathbf{D}f)\\
\text{subject to}\quad
&\mathbf{D}\in\bigcap_{e\in\cE}\argmax_{\mathbf{A}\in\cD(f)}J_e(\mathbf{A}f).
\end{aligned}
\label{eq:icr}
\end{equation}
The same operator must be globally optimal on each source distribution, while the outer encoder remains learnable. This is an exact population constraint: the common operator must attain each environment's maximum, rather than merely give similar coding values. It therefore lets us examine the limitation of the requirement itself without errors from finite samples or approximate inner optimisation. The operator domain is given by Equation~\eqref{eq:operator-domain}.

\begin{theorem}[A shared optimal coding operator admits the failing encoder]
\label{thm:shared}
In model~\eqref{eq:model}, let $\mathbf{D}=[\mathbf{u},\mathbf{v}]\in\cD(f_S)$ and $t=1-(\mathbf{u}^\top \mathbf{v})^2$. Then
\begin{equation}
J_e(\mathbf{D}f_S)=\frac12\log\!\left(
\frac{1+a+(a^2/4)t}{1+a+a^2\delta_e(1-\delta_e)t}\right).
\label{eq:operator-value}
\end{equation}
If $\delta_e\neq1/2$, its maximisers are exactly the operators with orthogonal unit columns. If $\delta_e=1/2$, every admissible operator is a maximiser. Consequently:
\begin{enumerate}
\item $(f_S,\mathbf{I}_2)$ satisfies the shared-optimality constraint for any collection of source mismatch probabilities, including distinct ones.
\item $J^*_{\mathrm{ICR}}=\Jstar$, attained by $(f_C,\mathbf{I}_2)$.
\item In the positive-noise construction of Theorem~\ref{thm:near}, $(f_S,\mathbf{I}_2)$ is a feasible solution whose objective is within $\eta$ of $J^*_{\mathrm{ICR}}$, yet its source-optimal classifier has target error at least $1-\eta$.
\end{enumerate}
The identity operator is also optimal on the reversed target environment.
\end{theorem}

\paragraph{Why shared optimality does not ensure stable prediction.}
Theorem~\ref{thm:shared} shows that the same coding operator is optimal for the environmental encoder $f_S$ in every training environment, even when the background--label correlations differ. It also remains optimal after the correlation reverses. However, the classifier still assigns labels according to the training association: when the environmental feature changes from usually matching the label to usually opposing it, these predictions become mostly incorrect.
Thus, an operator can remain optimal for the \mcr{} objective even though the relationship between the representation and the label has changed. The shared-optimality constraint does not rule out this failure. With positive noise, the environmental encoder can be arbitrarily close to maximising the constrained training objective while admitting an exactly optimal shared operator and a target prediction error arbitrarily close to one.

\subsection{Implications for incorporating invariance}
Our analysis takes a first step towards understanding whether the invariance principle can address the OOD limitations of \mcr{}. We consider a natural adaptation: requiring the same coding operator to be optimal across training environments. Yet even when this requirement holds exactly and the training objective is arbitrarily close to optimal, prediction can fail on almost every target sample.
The reason is that the coding operator can remain optimal even when the relationship between the environmental feature and the label reverses. The constraint favours well-separated representation directions, but does not ensure that those directions retain the same predictive meaning across environments.

It is worth noting that this result does not mean that invariance principle has no role to play in \mcr{}. It just shows that simply transferring the shared-optimality requirement widely used in previous work to the \mcr{} objective is not enough. To guarantee reliable OOD generalisation with \mcr{}, future approaches need additional new assumptions, learning principles, or regularisation that link representation geometry to stable predictive relationships across environments.

%% file: sections/discussion.tex
\section{Discussion and scope}
\label{sec:discussion}

\paragraph{What is specific to the \mcr{} objective?}
Distinguishing stable from unstable correlations is a challenge for many learning methods. Our analysis shows precisely how this problem arises for \mcr{}: we calculate the global optimum, the environmental encoder's gap from it, and all of its optimal coding operators. In particular, reversing the environmental correlation leaves the entire operator objective unchanged. This explains why a representation can perform well under the \mcr{} objective while supporting unreliable predictions. 

\paragraph{The same input support does not mean a small shift.}
In the identical-support construction, every target input can also occur during training. However, its probability can change substantially. For a source mismatch probability $\delta$ and its reversal, the target-to-source density ratio on mismatched inputs is $(1-\delta)/\delta$, which grows without bound as the objective gap approaches zero. Our result therefore does not rule out guarantees that also limit the density ratio or otherwise control the size of the shift. It shows that coding near-optimality alone, with coding precision fixed, cannot guarantee reliable prediction across the full family of shifts considered here.

\paragraph{Near-optimality and the solutions found in practice.}
With positive source noise, the stable encoder achieves a strictly higher objective than the environmental encoder. Indeed, in our finite model with full source support, every exact global optimiser admits zero source error, and any unrestricted source-optimal classifier also achieves zero target error; see Appendix~\ref{app:exact-stability}. Our negative result concerns representations that are arbitrarily close to optimal. Even with exact shared coding optimality, this closeness alone does not guarantee reliable target prediction across the shifts considered here. Whether gradient descent actually selects such representations is a separate question that our analysis does not answer. The exact-optimum failure at zero noise relies on a change in input support.

\paragraph{What additional information could help?}
Prediction is more stable when each class retains a similar representation distribution across environments. To make this precise, let $P$ and $Q$ have the same label prior $p_y$. For a fixed encoder $f$, suppose
\begin{equation}
	\operatorname{TV}\big(P(f(\mathbf{X})\mid Y=y),Q(f(\mathbf{X})\mid Y=y)\big)\leq\rho_y,
	\label{eq:tvassumption}
\end{equation}
where $\operatorname{TV}$ denotes total variation distance, a measure of the difference between two probability distributions. Thus, $\rho_y$ bounds how much the representation distribution of class $y$ changes across environments. Every fixed classifier $g$ then satisfies
\begin{equation}
	|\Risk_P(g\circ f)-\Risk_Q(g\circ f)|\leq\sum_y p_y\rho_y.
	\label{eq:tvbound}
\end{equation}
In other words, small changes in each class's representation distribution ensure a small change in classification error; see Appendix~\ref{app:tv} for the proof. However, equal coding rates alone do not ensure this stability: in our reversal example, the conditional total variation for each class is $|1-2\delta|$, which approaches one as $\delta$ approaches zero. Additional assumptions or information, such as paired interventions or known label-preserving transformations, may help establish the stability that the \mcr{} objective alone does not guarantee.

\paragraph{Scope of the model.}
Our construction uses two classes and two-dimensional unit-norm representations to make the failure mechanism explicit and allow exact analysis. Even in this basic setting, failure occurs despite the availability of an observed stable feature that perfectly predicts the label. This counterexample shows that the \mcr{} objective alone cannot guarantee reliable OOD generalisation. The two-dimensional setting makes this limitation analytically tractable, but does not imply that the underlying problem is confined to two dimensions. Besides, our shared-optimal-coding result shows that directly applying the invariance principle through a common optimal coding operator is insufficient, while leaving open other possible constraints that explicitly address predictive stability.

\section{Conclusion}
We studied the limits of \mcr{} for OOD generalisation and showed that well-separated representations alone do not ensure reliable prediction across environments. An encoder relying on an environmental feature can achieve the global optimum yet fail on every target sample when the feature--label correlation reverses and input support changes. Even when training and test inputs share the same support, the encoder can be arbitrarily close to optimal while making incorrect predictions on almost every target sample, despite the availability of a perfectly stable feature.
Requiring the same coding operator to be optimal across training environments does not resolve this problem: the operator can remain optimal even when the relationship between representations and labels reverses. Finite-sample experiments support these findings. We summarise that for \mcr{} to guarantee reliable OOD generalisation, additional new assumptions or learning principles are needed to ensure that predictive relationships remain stable across environments.

%% file: sections/statements.tex
\section*{Reproducibility statement}
The main text specifies the model, encoder class, and operator domain used in our theoretical analysis. Complete proofs and boundary cases are provided in the appendices, and Figure~\ref{fig:mechanism} illustrates the exact construction. Appendix~\ref{sec:numerical-validation} gives the sample sizes, parameter settings, and repetition counts for the numerical experiments, which evaluate prescribed encoders without training neural networks. Appendix~\ref{app:waterbirds-setup} describes how we train and evaluate the encoder on Waterbirds. Code for reproducing the experiments and generating the figures are included in the supplementary material.

\section*{AI use statement}
 AI tools (OpenAI  GPT-5 Sol and GPT-6 Astra) were used to polish the writing language, check mathematical proofs, and refine the \LaTeX{} source.

%% file: sections/appendix_overview.tex
\section*{Appendix overview}
The appendices provide background, proofs, and experimental details. The outline below indicates where to find each part and how it connects to the main text.

\medskip
\noindent\textbf{Appendix~\ref{sec:related}: Related work}\hfill p.~\pageref{sec:related}
\par\smallskip\noindent
Reviews \mcr{} and white-box models, followed by invariance-based approaches to OOD generalisation. It provides background for the motivation in Section~\ref{sec:intro}, the \mcr{} framework in Section~\ref{sec:setup}, and the invariance formulation in Section~\ref{sec:invariance}.

\medskip
\noindent\textbf{Appendix~\ref{app:proofs}: Proofs of the main results}\hfill p.~\pageref{app:proofs}
\par\smallskip\noindent
The proofs correspond to the following results and discussions in the main text:
\begin{center}
\small
\renewcommand{\arraystretch}{1.2}
\begin{tabular}{@{}p{0.13\linewidth}p{0.40\linewidth}p{0.40\linewidth}@{}}
\toprule
Appendix & Content & Corresponding main-text result \\
\midrule
\ref{app:bound} & Global \mcr{} optimum & Lemma~\ref{lem:bound}, Section~\ref{sec:failure} \\
\ref{app:exact} & Environmental encoder's objective and risk & Proposition~\ref{prop:exact}, Section~\ref{sec:failure} \\
\ref{app:near} & Near-optimality with identical support; noiseless boundary & Theorem~\ref{thm:near} and Section~\ref{sec:reversal} \\
\ref{app:exact-stability} & Exact optimisation under common support & Exact versus near-optimal solutions, Section~\ref{sec:discussion} \\
\ref{app:shared} & Shared optimal coding operator & Theorem~\ref{thm:shared}, Section~\ref{sec:invariance} \\
\ref{app:tv} & Prediction-risk bound from class-conditional stability & Equation~\eqref{eq:tvbound}, Section~\ref{sec:discussion} \\
\bottomrule
\end{tabular}
\end{center}

\medskip
\noindent\textbf{Appendix~\ref{app:additional}: Additional consequences and distinctions}\hfill p.~\pageref{app:additional}
\par\smallskip\noindent
Expands on Sections~\ref{sec:setup}--\ref{sec:discussion}: the information retained by second moments, a limitation of odd encoders, pooling versus averaging source objectives, the distinction between shared classification decisions and shared probability estimates, and the domain of admissible coding operators.

\medskip
\noindent\textbf{Appendix~\ref{sec:numerical-validation}: Numerical validation}\hfill p.~\pageref{sec:numerical-validation}
\par\smallskip\noindent
Checks the theoretical predictions using finite samples:
\begin{itemize}
\setlength{\itemsep}{2pt}
\setlength{\parskip}{0pt}
\item Figure~\ref{fig:numerical-reversal}: correlation reversal, corresponding to Proposition~\ref{prop:exact} and Section~\ref{sec:reversal}.
\item Figures~\ref{fig:numerical-gap} and~\ref{fig:numerical-sample-size}: near-optimality with identical support and the effect of sample size, corresponding to Theorem~\ref{thm:near}.
\item Figure~\ref{fig:numerical-shared}: shared operator optimality despite target failure, corresponding to Theorem~\ref{thm:shared}.
\item Table~\ref{tab:numerical-boundary}: the exact-optimum and near-optimum failure regimes discussed in Section~\ref{sec:failure}.
\end{itemize}

\medskip
\noindent\textbf{Appendix~\ref{app:waterbirds-setup}: Waterbirds experimental setup}\hfill p.~\pageref{app:waterbirds-setup}
\par\smallskip\noindent
Documents the data splits, preprocessing, encoder training, linear-probe evaluation, and saliency computation underlying the Waterbirds study in Section~\ref{sec:setup}, Table~\ref{tab:waterbirds-test}, and Figure~\ref{fig:MCR2-spurious}.
\clearpage

%% file: sections/related.tex
\section{Related work}
\label{sec:related}

\subsection{Maximal coding rate reduction and white-box models}

Coding-rate methods connect the structure of data with the number of bits needed to represent them. \citet{ma2007segmentation} use lossy coding to segment multivariate data by minimising their total coding length. Building on this idea, \citet{yu2020learning} introduce \mcr{} as an objective that encourages diversity within each class and separation between classes. Under suitable dimensional and precision conditions, their geometric analysis characterises orthogonal class subspaces and within-class dimensionality. \citet{wang2024geometry} further study the critical points and optimisation landscape of a regularised coding-rate formulation, providing a global account of its representation geometry. To reduce the computational cost of coding-rate optimisation, \citet{baek2022efficient} develop variational formulations that avoid repeated evaluation of class-wise log-determinants.

The emphasis on low-dimensional structure is supported by studies of intrinsic dimension~\citep{hein2005intrinsic,pope2021intrinsic} and low-dimensional models for high-dimensional data~\citep{wright2022high}. Related analyses of kernel learning connect learning curves to the smoothness and effective dimension of the data~\citep{spigler2020asymptotic}. These results help explain why representation geometry matters for learning, but do not establish that the represented features remain predictive under environmental changes. In a different setting, \citet{locatello2019challenging} show that unsupervised disentanglement requires assumptions about the model and data. Their result concerns the identification of latent factors; ours concerns the stability of label relationships under a supervised coding objective.

The coding-rate perspective also provides a basis for white-box network design, following the broader approach of deriving interpretable networks by unrolling optimisation algorithms~\citep{monga2021algorithm}. \citet{chan2022redunet} derive ReduNet by unfolding the optimisation of coding-rate reduction into network layers. \citet{yu2024white} develop sparse rate reduction and introduce CRATE, a white-box Transformer whose components have explicit roles in compressing and sparsifying representations. Later work improves the scalability of CRATE through CRATE-$\alpha$~\citep{yang2024scaling} and derives Attention-only Transformers by unfolding subspace denoising~\citep{wang2025attention}. Coding-rate principles have also been applied to incremental learning~\citep{wu2021incremental,tong2023incremental}, joint discriminative and generative learning~\citep{dai2022ctrl,tong2024unsupervised}, image clustering~\citep{chu2024image}, and self-supervised learning~\citep{wu2025simplifying}.

These studies establish the value of coding-rate principles for learning structured representations and designing interpretable architectures. Our work examines the predictive reliability of these representations when the environment changes. In domain generalisation, \citet{atghaei2026consistency} combine coding-rate reduction with classification and feature-consistency objectives. Our analysis clarifies what the coding objective itself guarantees, as well as the limits of a direct shared-optimal-coding extension. The architectural developments provide context for the importance of coding-rate methods; our OOD results concern the objective and the specified coding constraint. This distinction helps identify the additional information needed for reliable prediction across environments.

\subsection{The invariance principle and OOD generalisation}

Domain generalisation studies how to learn from source environments and predict in unseen ones~\citep{zhou2023domain}. Models can learn environmental cues that predict labels during training but become unreliable in a new environment~\citep{geirhos2020shortcut,sagawa2020distributionally}. \citet{sagawa2020investigation} study how overparameterisation can worsen performance on groups where spurious correlations do not hold. \citet{nagarajan2021understanding} identify geometric and statistical factors that lead gradient-based learning to exploit such correlations. These studies examine how training produces unstable predictors; our analysis asks whether the coding objective and its shared-optimality constraint exclude them.

Causal approaches connect generalisation to assumptions about which mechanisms remain stable when data distributions change~\citep{schoelkopf2012causal}. \citet{muandet2013domain} learn domain-invariant features while preserving predictive information. \citet{rojas2018invariant} study transfer through subsets of predictors whose relationship with the target is invariant across tasks. Anchor regression provides robustness against specified distribution shifts by using observed exogenous variables~\citep{rothenhaeusler2021anchor}. These approaches make explicit the information or assumptions used to support transfer.

The invariance principle seeks predictive relationships that remain valid across environments. \citet{peters2016causal} formalise invariant prediction for causal inference, linking stable predictive relationships to causal structure under suitable assumptions. Building on this principle, \citet{arjovsky2019invariant} propose IRM, which learns a representation for which the same predictor is optimal in every training environment. REx instead encourages similar prediction risks across environments, aiming to improve robustness to shifts beyond those observed during training~\citep{krueger2021out}. Related work formulates invariant learning as a game among environments~\citep{ahuja2020invariant}.

Several studies examine the conditions needed for these ideas to support OOD generalisation. \citet{rosenfeld2021risks} show that IRM constraints can allow predictors based on environmental features that fail in sufficiently different test environments. \citet{kamath2021does} show that practical IRM formulations can fail to capture the intended invariances, even in simple population settings. \citet{ahuja2021invariance} identify limitations of invariance alone in classification and show how an information-bottleneck constraint can help under their assumptions. Empirical studies also highlight the importance of evaluation and model selection when comparing domain-generalisation methods~\citep{gulrajani2021search}. WILDS documents naturally occurring distribution shifts across diverse applications~\citep{koh2021wilds}, while \citet{wiles2022fine} compare methods across different types of shift and show that their relative performance depends on the setting.

Our work brings this analysis to the coding-rate framework. We study a direct formulation in which the invariance principle translates into a shared coding operator that is optimal across training environments. The same operator can satisfy this requirement even when the relationship between the representation and the label reverses. This result distinguishes agreement on coding geometry from stability of predictive relationships, and identifies a limit of this particular way of introducing invariance into \mcr{}.

%% file: sections/proofs.tex
\section{Proofs of the main results}
\label{app:proofs}
This appendix gives the complete derivations of the coding bound, the environmental encoder's objective and risk, the identical-support construction, and the shared-optimal-coding result. Throughout the main construction, $a=2/\epsilon^2>0$ is fixed. All encoders and classifiers are measurable; measurability is automatic for functions on the finite input domain used in the construction.

\subsection{Sharp binary coding bound}
\label{app:bound}
\begin{proof}[Proof of Lemma~\ref{lem:bound}]
Fix an environment and abbreviate $\mathbf{M}_y=\mathbf{M}_{e,y}(f)$ and $\mathbf{M}=(\mathbf{M}_0+\mathbf{M}_1)/2$. The second identity uses the balanced label prior. We first establish separate bounds on the expansion and compression terms, then identify when both bounds are attained.

\paragraph{Step 1: Properties of the second moments.}
For every $\mathbf{q}\in\R^2$,
\begin{equation}
\mathbf{q}^\top\mathbf{M}_y\mathbf{q}
=\E_e[(\mathbf{q}^\top f(\mathbf{X}))^2\mid Y=y]\geq0.
\end{equation}
Thus each conditional moment is symmetric and positive semidefinite. Linearity of trace and the unit-norm constraint give
\begin{equation}
\tr(\mathbf{M}_y)
=\E_e[\tr(f(\mathbf{X})f(\mathbf{X})^\top)\mid Y=y]
=\E_e[\|f(\mathbf{X})\|_2^2\mid Y=y]=1.
\end{equation}
The same properties hold for their average $\mathbf{M}$. In particular, all matrices inside the log-determinants are positive definite because $a>0$, so the objective is finite.

\paragraph{Step 2: Bound on the expansion term.}
Let $\lambda_1,\lambda_2$ be the eigenvalues of $\mathbf{M}$. They are nonnegative and sum to one. The spectral decomposition and the arithmetic--geometric mean inequality imply
\begin{align}
\det(\mathbf{I}_2+a\mathbf{M})
&=(1+a\lambda_1)(1+a\lambda_2)\notag\\
&\leq\left(\frac{(1+a\lambda_1)+(1+a\lambda_2)}2\right)^2\notag\\
&=\left(1+\frac a2(\lambda_1+\lambda_2)\right)^2
=(1+a/2)^2.
\label{eq:proof-global}
\end{align}
Consequently, $\tfrac12\log\det(\mathbf{I}_2+a\mathbf{M})\leq\log(1+a/2)$. Since $a>0$, equality holds precisely when $\lambda_1=\lambda_2=1/2$. A symmetric matrix with these eigenvalues is $\mathbf{I}_2/2$, so this equality condition is independent of the choice of eigenbasis.

\paragraph{Step 3: Bound on each compression term.}
For a two-dimensional matrix, the determinant expansion gives
\begin{equation}
\det(\mathbf{I}_2+a\mathbf{M}_y)
=1+a\tr(\mathbf{M}_y)+a^2\det(\mathbf{M}_y)
=1+a+a^2\det(\mathbf{M}_y)\geq1+a.
\label{eq:proof-conditional}
\end{equation}
The last inequality follows because both eigenvalues of $\mathbf{M}_y$ are nonnegative. Equality holds if and only if its determinant is zero. Together with trace one, this forces its eigenvalues to be $1$ and $0$. Equivalently,
\begin{equation}
\mathbf{M}_y=\mathbf{v}_y\mathbf{v}_y^\top,
\qquad \|\mathbf{v}_y\|_2=1.
\end{equation}
Thus equality in the conditional bound means that the second moment is a rank-one orthogonal projector.

\paragraph{Step 4: Combine the bounds and characterise equality.}
With the two class weights equal to $1/2$, the objective is
\begin{equation}
J_e(f)=\frac12\log\det(\mathbf{I}_2+a\mathbf{M})
-\frac14\sum_{y=0}^1\log\det(\mathbf{I}_2+a\mathbf{M}_y).
\end{equation}
Applying the preceding inequalities gives $J_e(f)\leq\log(1+a/2)-\tfrac12\log(1+a)=\Jstar$. To verify that no slack can cancel, write
\begin{align}
\Jstar-J_e(f)
&=\left[\log(1+a/2)-\frac12\log\det(\mathbf{I}_2+a\mathbf{M})\right]\notag\\
&\quad+\frac14\sum_{y=0}^1
\left[\log\det(\mathbf{I}_2+a\mathbf{M}_y)-\log(1+a)\right].
\end{align}
Every bracket is nonnegative. The total is zero if and only if every bracket is zero. Hence equality requires $\mathbf{M}=\mathbf{I}_2/2$ and $\mathbf{M}_y=\mathbf{v}_y\mathbf{v}_y^\top$ for both labels. These conditions imply
\begin{equation}
\mathbf{v}_0\mathbf{v}_0^\top+\mathbf{v}_1\mathbf{v}_1^\top=\mathbf{I}_2.
\end{equation}
Taking the quadratic form along $\mathbf{v}_0$ yields
\begin{equation}
1+(\mathbf{v}_0^\top\mathbf{v}_1)^2=1,
\end{equation}
so $\mathbf{v}_0^\top\mathbf{v}_1=0$. Conversely, two orthogonal unit vectors form an orthonormal basis of $\R^2$. Their projectors sum to $\mathbf{I}_2$, and both the expansion and compression bounds are attained. This proves the stated necessary and sufficient equality condition.

\paragraph{Step 5: Attainment in every training environment.}
In model~\eqref{eq:model}, $C=Y$ almost surely, so $f_C(\mathbf{X})=\mathbf{b}_Y$. Its conditional moments are $\mathbf{b}_y\mathbf{b}_y^\top$ in every environment, regardless of $\delta_e$. Therefore $J_e(f_C)=\Jstar$ for every $e$. For any other $f\in\cF$,
\begin{equation}
J_{\mathrm{tr}}(f)=\sum_e w_eJ_e(f)\leq\sum_e w_e\Jstar=\Jstar.
\end{equation}
The encoder $f_C$ belongs to $\cF$ and attains this bound. Thus the bound is the global maximum over the full encoder class, rather than only over the two encoders displayed in the construction.
\end{proof}

\subsection{Exact coding value and source-optimal head}
\label{app:exact}
\begin{proof}[Proof of Proposition~\ref{prop:exact}]
Fix an environment with mismatch probability $\delta\in[0,1]$. We calculate the complete distribution on the two output codes before evaluating the objective and prediction risk.

\paragraph{Step 1: Conditional and marginal code distributions.}
Since $S=Y\oplus B_e$ and $B_e$ is independent of $Y$,
\begin{equation}
P_e(S=s\mid Y=y)=
\begin{cases}
1-\delta,&s=y,\\
\delta,&s\neq y.
\end{cases}
\end{equation}
The balanced label prior implies, for each $s\in\{0,1\}$,
\begin{equation}
P_e(S=s)=\frac12(1-\delta)+\frac12\delta=\frac12.
\end{equation}
Because $f_S(\mathbf{X})=\mathbf{b}_S$, the representation marginal is therefore uniform on the two codes for every value of $\delta$, including the endpoints. The conditional and global moments are
\begin{equation}
\mathbf{M}_0=\begin{pmatrix}1-\delta&0\\0&\delta\end{pmatrix},\qquad
\mathbf{M}_1=\begin{pmatrix}\delta&0\\0&1-\delta\end{pmatrix},\qquad
\mathbf{M}=\frac12\mathbf{I}_2.
\end{equation}

\paragraph{Step 2: Evaluate the log-determinants.}
The global determinant is $(1+a/2)^2$. Each conditional determinant is
\begin{align}
(1+a\delta)(1+a(1-\delta))
&=1+a\delta+a(1-\delta)+a^2\delta(1-\delta)\notag\\
&=1+a+a^2\delta(1-\delta).
\end{align}
Substituting into the population objective and combining the two identical conditional terms gives
\begin{align}
J_\delta(f_S)
&=\log(1+a/2)-\frac12\log\big(1+a+a^2\delta(1-\delta)\big)\notag\\
&=\log(1+a/2)-\frac12\log(1+a)
-\frac12\log\left(1+\frac{a^2}{1+a}\delta(1-\delta)\right)\notag\\
&=\Jstar-\gap(\delta).
\end{align}
Here we factored $1+a$ from the conditional determinant; the factor is strictly positive. The formula holds on the entire interval $[0,1]$. In particular, the gap vanishes at $\delta=0$ and $\delta=1$, is strictly positive for $0<\delta<1$, and is unchanged by replacing $\delta$ with $1-\delta$.

\paragraph{Step 3: Risk of the fixed classifier.}
The classifier $g_0$ reads the code index, so $g_0(f_S(\mathbf{X}))=S$. Its error event is exactly $\{B_e=1\}$, giving
\begin{equation}
\Risk_e(g_0\circ f_S)=P_e(S\neq Y)=P_e(B_e=1)=\delta_e.
\end{equation}
For the stable encoder, $g_0(f_C(\mathbf{X}))=C=Y$ almost surely. Its risk is zero in every source and target environment.

\paragraph{Step 4: Optimality on the weighted source mixture.}
Let $P_{\mathrm{tr}}=\sum_e w_eP_e$ and $\bar\delta=\sum_e w_e\delta_e$. Minimising the weighted source risk is equivalent to minimising risk under this mixture. For either code index $s$,
\begin{align}
P_{\mathrm{tr}}(Y=s,S=s)&=\frac12\sum_e w_e(1-\delta_e)=\frac12(1-\bar\delta),\notag\\
P_{\mathrm{tr}}(Y=1-s,S=s)&=\frac12\sum_e w_e\delta_e=\frac12\bar\delta.
\end{align}
As $P_{\mathrm{tr}}(S=s)=1/2$, Bayes' rule gives
\begin{equation}
P_{\mathrm{tr}}(Y=s\mid S=s)=1-\bar\delta,
\qquad P_{\mathrm{tr}}(Y=1-s\mid S=s)=\bar\delta.
\end{equation}
Predicting $s$ at $\mathbf{b}_s$ has conditional error $\bar\delta$, whereas predicting $1-s$ has conditional error $1-\bar\delta$. If $\bar\delta<1/2$, the former is strictly smaller at both codes. More explicitly, for any deterministic binary classifier $g$, let $m(g)$ be the number of codes on which $g(\mathbf{b}_s)\neq s$. Then
\begin{equation}
\Risk_{\mathrm{tr}}(g\circ f_S)
=\bar\delta+\frac{1-2\bar\delta}{2}m(g).
\end{equation}
This is uniquely minimised on the representation support by $m(g)=0$, which is exactly $g_0$. Allowing randomised decisions does not lower the minimum: their conditional errors are convex combinations of the two errors just calculated. Values of $g$ outside the two-code support do not affect its risk. Finally, $g_0(\mathbf{z})=\mathbf1\{z_2>z_1\}$ implements the optimal assignments on the basis codes, so a linear decision boundary attains the unrestricted optimum.
\end{proof}

\subsection{Near-optimality with identical input support}
\label{app:near}
\begin{proof}[Proof of Theorem~\ref{thm:near}]
We first establish the quantitative gap estimate for arbitrary small source mismatch probabilities. We then give an explicit choice of environments and verify all requirements of the theorem.

\paragraph{Step 1: Bound the coding gap.}
For $x\geq0$, $\log(1+x)\leq x$. Applying this with $x=a^2\delta(1-\delta)/(1+a)$ yields
\begin{equation}
0\leq\gap(\delta)\leq\frac{a^2}{2(1+a)}\delta(1-\delta)
\leq\frac{a^2}{2(1+a)}\delta,
\qquad 0\leq\delta\leq1.
\end{equation}
By Lemma~\ref{lem:bound} and Proposition~\ref{prop:exact},
\begin{align}
\max_{f\in\cF}J_{\mathrm{tr}}(f)-J_{\mathrm{tr}}(f_S)
&=\Jstar-\sum_e w_e[\Jstar-\gap(\delta_e)]\notag\\
&=\sum_e w_e\gap(\delta_e)\notag\\
&\leq\frac{a^2}{2(1+a)}\sum_e w_e\delta_e\notag\\
&\leq\frac{a^2}{2(1+a)}\delta_{\max}.
\end{align}
The final step uses positive weights summing to one and $\delta_e\leq\delta_{\max}$. This proves Equation~\eqref{eq:gapbound} without requiring equal source weights.

\paragraph{Step 2: Choose the source and target environments.}
Fix the prescribed $a>0$, finite positive integer $E$, source weights, and $\eta>0$. Define
\begin{equation}
r=\min\left\{\frac14,\frac\eta2,\frac{(1+a)\eta}{a^2}\right\}>0,
\qquad
\delta_e=\frac{e}{E}r\quad(e=1,\ldots,E),\qquad
\delta_{\mathrm{te}}=1-r.
\end{equation}
All three terms defining $r$ are strictly positive. Hence $0<r\leq1/4$, and the source parameters satisfy
\begin{equation}
0<\delta_1<\cdots<\delta_E=r<\frac12.
\end{equation}
For $E=1$, this simply means $\delta_1=r$. The target parameter lies in $[3/4,1)$, so its environmental feature is more often incorrect than correct. The largest source mismatch probability is $\delta_{\max}=r$.

\paragraph{Step 3: Verify objective quality and source optimality.}
The gap estimate and $r\leq(1+a)\eta/a^2$ give
\begin{equation}
0\leq\max_{f\in\cF}J_{\mathrm{tr}}(f)-J_{\mathrm{tr}}(f_S)
\leq\frac{a^2r}{2(1+a)}\leq\frac\eta2\leq\eta.
\end{equation}
Also, $\bar\delta=\sum_e w_e\delta_e\leq r<1/2$. Proposition~\ref{prop:exact} therefore establishes that $g_0$ is source-optimal on $f_S$. No target labels are used to select this classifier.

\paragraph{Step 4: Verify target failure and the stable alternative.}
Keeping the same classifier at test time gives
\begin{equation}
\Risk_{\mathrm{te}}(g_0\circ f_S)=\delta_{\mathrm{te}}=1-r
\geq1-\frac\eta2\geq1-\eta.
\end{equation}
At the same time, $\Risk_{\mathrm{te}}(g_0\circ f_C)=0$. The failure thus occurs despite the availability of an encoder and classifier that are perfectly accurate in every environment.

\paragraph{Step 5: Verify common support.}
For an environment with parameter $\delta$, the input probabilities are
\begin{equation}
\begin{array}{c|cccc}
(c,s)&(0,0)&(0,1)&(1,0)&(1,1)\\\hline
P_e(C=c,S=s)&(1-\delta)/2&\delta/2&\delta/2&(1-\delta)/2.
\end{array}
\end{equation}
Each source has $0<\delta_e<1$ and the target has $0<1-r<1$. Every entry is therefore strictly positive in every environment, and every input support equals $\{0,1\}^2$. Since $Y=C$ deterministically, the joint support of $(\mathbf{X},Y)$ is also identical: it consists of the four triples $((c,s),c)$. The theorem concerns a change in probabilities and predictive relationships on this shared support, rather than the appearance of previously impossible inputs.

These checks establish the two inequalities, source optimality, the stable zero-risk alternative, and common support for any specified finite number of sources and any specified positive weights. The encoder $f_S$ itself is the same throughout; only the environment parameters are chosen as a function of the requested tolerance.
\end{proof}

\paragraph{The exact-optimum boundary.}
If every source has $\delta_e=0$, then $f_C=f_S$ almost surely on every source, and both attain $\Jstar$. A target with $\delta_{\mathrm{te}}=1$ has $S=1-Y$, so $g_0\circ f_S$ is wrong with probability one. The source input support is $\{(0,0),(1,1)\}$, while the target support is $\{(0,1),(1,0)\}$. This is why the exact boundary example cannot be used as a proof of exact-optimum failure with identical input support. In the positive-noise construction, $\gap(\delta_e)>0$ for every source, so $f_C$ is strictly better in objective value than $f_S$. The theorem establishes the insufficiency of a small objective gap, not inevitable failure of an exact optimiser.

\paragraph{Why precision is fixed in the theorem.}
The coefficient $a^2/[2(1+a)]$ depends on coding precision. For a fixed $\delta\in(0,1)$, the exact gap is $\tfrac12\log(1+a^2\delta(1-\delta)/(1+a))$, which grows without bound as $a\to\infty$. The theorem instead fixes any $a>0$ and then selects sufficiently small source mismatch probabilities. Its quantifier order is therefore important: it does not assert uniform near-optimality in precision at a fixed noise level.

\subsection{Exact optimisation under common support}
\label{app:exact-stability}
The positive-noise result concerns representations arbitrarily close to the optimum. The following consequence of Lemma~\ref{lem:bound} makes the distinction from exact optimisation explicit.

\begin{proposition}[Exact optimisers in the full-support model]
\label{prop:exact-stability}
In model~\eqref{eq:model}, suppose $0<\delta_e<1$ in every training environment and all source weights are positive. If $f\in\cF$ attains $J_{\mathrm{tr}}(f)=\Jstar$, then every source-optimal classifier over all functions on the representation has zero risk in every target environment from the same model, for any $\delta_{\mathrm{te}}\in[0,1]$. The same statement holds for the representation $\mathbf{D}f$ of any feasible pair attaining $J^*_{\mathrm{ICR}}$.
\end{proposition}
\begin{proof}
We use the equality condition of the coding bound to identify which representations an exact optimiser can assign to each input.

\paragraph{Step 1: Equality holds in each source.}
By Lemma~\ref{lem:bound}, $\Jstar-J_e(f)\geq0$ for every $e$. Since all $w_e>0$ and
\begin{equation}
0=\Jstar-J_{\mathrm{tr}}(f)=\sum_e w_e[\Jstar-J_e(f)],
\end{equation}
each term must be zero. Choose any source environment. The equality characterisation gives unit vectors $\mathbf{v}_0,\mathbf{v}_1$ such that $\mathbf{v}_0\perp\mathbf{v}_1$ and $\mathbf{M}_{e,y}(f)=\mathbf{v}_y\mathbf{v}_y^\top$.

\paragraph{Step 2: Every input of a class maps to its class line.}
Fix $y\in\{0,1\}$ and a unit vector $\mathbf{q}_y\perp\mathbf{v}_y$. Since $C=Y$, the conditional moment identity implies
\begin{equation}
0=\mathbf{q}_y^\top\mathbf{M}_{e,y}(f)\mathbf{q}_y
=\sum_{s=0}^1 P_e(S=s\mid Y=y)
\big(\mathbf{q}_y^\top f(y,s)\big)^2.
\end{equation}
Both conditional probabilities are strictly positive because $0<\delta_e<1$. Each summand is nonnegative, so both squares vanish. Thus $f(y,s)$ lies on the line spanned by $\mathbf{v}_y$. The unit-norm condition further gives
\begin{equation}
f(y,s)\in\{\mathbf{v}_y,-\mathbf{v}_y\},\qquad s\in\{0,1\}.
\end{equation}
The two class lines are orthogonal, so their unit-vector sets are disjoint. The encoder may still use $s$ to choose a sign, but that choice cannot move an input onto the other class line.

\paragraph{Step 3: Source-optimal predictions are correct on all possible inputs.}
A classifier assigning label $y$ to each attained code in $\{\mathbf{v}_y,-\mathbf{v}_y\}$ has zero source risk. Hence the optimal source risk over all classifiers is zero. Every input in $\{0,1\}^2$ has positive source probability, so any source-optimal classifier must classify the representation of every one of these inputs correctly. Otherwise its source risk would be strictly positive.

A target environment in the stated family changes only the probabilities assigned to these inputs; it retains the deterministic label $Y=C$. The same classifier is therefore correct on every target input, including when $\delta_{\mathrm{te}}=0$ or $1$ reduces the target support. Its target risk is zero.

\paragraph{Step 4: Exact constrained optima.}
Theorem~\ref{thm:shared} establishes $J^*_{\mathrm{ICR}}=\Jstar$. For a feasible maximising pair, let $h(\mathbf{x})=\mathbf{D}f(\mathbf{x})$. Admissibility ensures $h\in\cF$, and attainment gives $\sum_e w_eJ_e(h)=\Jstar$. Applying the preceding steps to $h$ proves the final assertion.
\end{proof}

This proposition concerns an unrestricted source-optimal classifier. A linear classifier need not separate two classes that each occupy both signs of their respective lines. For the particular encoders $f_C$ and $f_S$, however, the linear rule in Section~\ref{sec:setup} attains the source optimum as stated. The proposition also uses the finite model and full source support; it does not extend the zero-risk conclusion to arbitrary input spaces or label mechanisms.

\subsection{Shared coding optimality}
\label{app:shared}
\begin{proof}[Proof of Theorem~\ref{thm:shared}]
We characterise the entire inner operator optimisation, then verify feasibility and objective quality in the outer problem.

\paragraph{Step 1: Characterise the admissible operators.}
Both $\mathbf{b}_0$ and $\mathbf{b}_1$ occur in the image of $f_S$ on the fixed input domain $\cX$. For a linear operator $\mathbf{D}=[\mathbf{u},\mathbf{v}]$, its outputs on these codes are $\mathbf{D}\mathbf{b}_0=\mathbf{u}$ and $\mathbf{D}\mathbf{b}_1=\mathbf{v}$. Thus the requirement $\|\mathbf{D}f_S(\mathbf{x})\|_2=1$ for all $\mathbf{x}\in\cX$ is equivalent to $\|\mathbf{u}\|_2=\|\mathbf{v}\|_2=1$. The same argument applies to $f_C$. This uses the full input domain specified in the operator constraint, including when an individual environment has a smaller support.

Write $\gamma=\mathbf{u}^\top\mathbf{v}$ and $t=1-\gamma^2$. Cauchy--Schwarz gives $0\leq t\leq1$. Every value in this interval is attainable: choosing $\mathbf{u}=(1,0)^\top$ and $\mathbf{v}=(\sqrt{1-t},\sqrt t)^\top$ produces unit columns with the required value of $t$. In particular, $t=0$ allows parallel or antiparallel columns, and $t=1$ means orthogonal columns.

\paragraph{Step 2: Compute the moments after applying the operator.}
For $p\in[0,1]$, define
\begin{equation}
\mathbf{A}_p=p\mathbf{u}\mathbf{u}^\top+(1-p)\mathbf{v}\mathbf{v}^\top
=\mathbf{D}\diag(p,1-p)\mathbf{D}^\top.
\end{equation}
It is positive semidefinite and has trace $p\|\mathbf{u}\|_2^2+(1-p)\|\mathbf{v}\|_2^2=1$. Its determinant is
\begin{align}
\det(\mathbf{A}_p)
&=p(1-p)(\det\mathbf{D})^2\notag\\
&=p(1-p)\det(\mathbf{D}^\top\mathbf{D})\notag\\
&=p(1-p)\det\begin{pmatrix}1&\gamma\\\gamma&1\end{pmatrix}
=p(1-p)t.
\end{align}
No inverse of $\mathbf{D}$ is used, so these identities remain valid for singular operators and for $p=0$ or $p=1$. The two-dimensional determinant identity now gives
\begin{equation}
\det(\mathbf{I}_2+a\mathbf{A}_p)=1+a+a^2p(1-p)t.
\end{equation}
The marginal code probabilities are $1/2,1/2$, so the global moment of $\mathbf{D}f_S$ is $\mathbf{A}_{1/2}$. Its class-conditional moments are $\mathbf{A}_{1-\delta_e}$ and $\mathbf{A}_{\delta_e}$. Their determinants are equal, although the matrices need not be equal.

\paragraph{Step 3: Evaluate the full operator objective.}
Substituting these moments into the population objective gives
\begin{align}
J_e(\mathbf{D}f_S)
&=\frac12\log(1+a+a^2t/4)
-\frac12\log(1+a+a^2\delta_e(1-\delta_e)t)\notag\\
&=\frac12\log\left(
\frac{1+a+(a^2/4)t}{1+a+a^2\delta_e(1-\delta_e)t}\right).
\end{align}
This proves Equation~\eqref{eq:operator-value}. In particular, all dependence on the two unit columns is through $t$. All logarithm arguments are positive, so differentiation is valid throughout $[0,1]$, with one-sided derivatives at its endpoints.

\paragraph{Step 4: Identify every maximising operator.}
Let $A=1+a$, $B=a^2/4$, and $C=a^2\delta_e(1-\delta_e)$. Here these are scalar constants. Direct differentiation gives
\begin{align}
\frac{\mathrm d}{\mathrm dt}\left[\frac12\log\frac{A+Bt}{A+Ct}\right]
&=\frac12\left(\frac{B}{A+Bt}-\frac{C}{A+Ct}\right)\notag\\
&=\frac{B(A+Ct)-C(A+Bt)}{2(A+Bt)(A+Ct)}\notag\\
&=\frac{A(B-C)}{2(A+Bt)(A+Ct)}.
\end{align}
Substituting these constants gives
\begin{equation}
\frac{\partial J_e(\mathbf{D}f_S)}{\partial t}
=\frac{a^2(1+a)\big(1/4-\delta_e(1-\delta_e)\big)}
{2\big(1+a+a^2t/4\big)\big(1+a+a^2\delta_e(1-\delta_e)t\big)}.
\label{eq:derivative}
\end{equation}
The denominator is strictly positive, and
\begin{equation}
\frac14-\delta_e(1-\delta_e)=\left(\delta_e-\frac12\right)^2.
\end{equation}
If $\delta_e\neq1/2$, the objective is strictly increasing in $t$. Because $t=1$ is attainable, its maximising operators are exactly those with orthogonal unit columns. If $\delta_e=1/2$, numerator and denominator inside the logarithm are identical, and $J_e(\mathbf{D}f_S)=0$ for every admissible operator. This proves both cases in the theorem, including the complete characterisation of the maximiser set.

\paragraph{Step 5: Verify common optimality, including under reversal.}
The identity $\mathbf{I}_2$ has orthogonal unit columns. It is therefore optimal in every environment with $\delta_e\neq1/2$ and is also a maximiser when $\delta_e=1/2$. Consequently, it belongs to the intersection of all source argmax sets, independently of the number or distinctness of their mismatch probabilities. If at least one source has $\delta_e\neq1/2$, that intersection is exactly the set of orthogonal operators; if all have $\delta_e=1/2$, it is the entire admissible operator set.

Moreover, replacing $\delta_e$ by $1-\delta_e$ leaves $\delta_e(1-\delta_e)$ unchanged. It therefore preserves the full objective function of $\mathbf{D}$, not merely its maximum value. In particular, $\mathbf{I}_2$ remains optimal on the reversed target distribution even though the source classifier's predictions become mostly incorrect there.

\paragraph{Step 6: Determine the outer constrained optimum.}
For the stable encoder, conditional on $Y=0$ the output after applying $\mathbf{D}$ is $\mathbf{u}$, and conditional on $Y=1$ it is $\mathbf{v}$. The preceding computation with $\delta=0$ gives
\begin{equation}
J_e(\mathbf{D}f_C)=\frac12\log\frac{1+a+a^2t/4}{1+a}.
\end{equation}
This function is strictly increasing in $t$ and does not depend on the environment. Hence $(f_C,\mathbf{I}_2)$ is feasible in Equation~\eqref{eq:icr}. At $t=1$, the numerator equals $(1+a/2)^2$, giving $J_e(f_C)=\Jstar$ in every source.

Conversely, for every admissible pair $(f,\mathbf{D})$, the composition $h(\mathbf{x})=\mathbf{D}f(\mathbf{x})$ is a unit-norm encoder into $\R^2$ on the whole input domain. Lemma~\ref{lem:bound} applies to $h$ in each balanced environment. Thus every feasible pair satisfies
\begin{equation}
\sum_e w_eJ_e(\mathbf{D}f)\leq\sum_e w_e\Jstar=\Jstar.
\end{equation}
Since the feasible pair $(f_C,\mathbf{I}_2)$ attains this bound, the constrained supremum is attained and $J^*_{\mathrm{ICR}}=\Jstar$.

\paragraph{Step 7: Verify the near-optimal failing feasible solution.}
Apply the positive-noise environments constructed in Theorem~\ref{thm:near}. Step 5 establishes exact shared optimality of $\mathbf{I}_2$ for $f_S$, so $(f_S,\mathbf{I}_2)$ is feasible. Its gap from the constrained optimum is
\begin{equation}
J^*_{\mathrm{ICR}}-\sum_e w_eJ_e(\mathbf{I}_2f_S)
=\Jstar-J_{\mathrm{tr}}(f_S)\leq\eta.
\end{equation}
Applying the identity does not change the representation seen by the classifier. The source-optimal rule remains $g_0$, and its target error remains at least $1-\eta$. All support and stable-feature properties are inherited from the same construction. This proves the three assertions and the target-optimality statement.
\end{proof}

\paragraph{Why the operator comparison is nontrivial.}
For example, $\mathbf{D}=[\mathbf{b}_0,\mathbf{b}_0]$ is admissible on the two-code support but maps both codes to the same vector. It has $t=0$ and coding-rate reduction zero. The identity operator has $t=1$ and strictly positive rate reduction whenever $\delta\neq1/2$. The shared-maximiser statement therefore compares genuinely different coding geometries. It is stronger than saying that all orthogonal coordinate changes preserve an objective, but it remains a statement about the explicit domain $\cD(f)$.

\subsection{A sufficient distributional condition for fixed-head stability}
\label{app:tv}
\begin{proof}[Proof of Equation~\eqref{eq:tvbound}]
Let $P$ and $Q$ have the same label prior $p_y$. For a fixed encoder $f$, denote the conditional representation laws by $\mu_y=P(f(\mathbf{X})\in\cdot\mid Y=y)$ and $\nu_y=Q(f(\mathbf{X})\in\cdot\mid Y=y)$. Labels of zero prior can be omitted. We use the convention
\begin{equation}
\operatorname{TV}(\mu,\nu)=\sup_{A}|\mu(A)-\nu(A)|,
\end{equation}
where the supremum ranges over measurable sets. For a fixed measurable classifier $g$, define the conditional error set $A_y=\{\mathbf{z}:g(\mathbf{z})\neq y\}$. The law of total probability gives
\begin{equation}
\Risk_P(g\circ f)=\sum_y p_y\mu_y(A_y),
\qquad \Risk_Q(g\circ f)=\sum_y p_y\nu_y(A_y).
\end{equation}
Subtracting these expressions, applying the triangle inequality, and then using the definition of total variation gives
\begin{align}
|\Risk_P(g\circ f)-\Risk_Q(g\circ f)|
&=\left|\sum_y p_y[\mu_y(A_y)-\nu_y(A_y)]\right|\notag\\
&\leq\sum_y p_y|\mu_y(A_y)-\nu_y(A_y)|\notag\\
&\leq\sum_y p_y\operatorname{TV}(\mu_y,\nu_y)\notag\\
&\leq\sum_y p_y\rho_y.
\end{align}
The last step is exactly the class-conditioned stability assumption in Equation~\eqref{eq:tvassumption}. The bound applies to any fixed classifier, without requiring source optimality. Equal label priors allow the same weights to appear in the two risk decompositions; matching a scalar coding value is not used in this derivation.

To check sharpness in the reversal construction, the two conditional measures for either class have masses $(1-\delta,\delta)$ and $(\delta,1-\delta)$ on the two basis codes, with their order depending on the class. Total variation on this two-point support is
\begin{equation}
\frac12\left(|(1-\delta)-\delta|+|\delta-(1-\delta)|\right)=|1-2\delta|.
\end{equation}
The fixed classifier $g_0$ has risks $\delta$ and $1-\delta$, whose absolute difference is also $|1-2\delta|$. With class weights $1/2,1/2$, the right-hand side of the bound is exactly the same quantity. Thus the bound is attained in this example, despite identical coding values before and after reversal.
\end{proof}

%% file: sections/secondary.tex
\section{Additional consequences and distinctions}
\label{app:additional}
\subsection{The second-moment information retained by the objective}
\begin{proposition}[Moment equivalence]
\label{prop:moment}
Fix the representation dimension, precision, and label prior in an environment. If two encoders $f,h$ have $\mathbf{M}_{e,y}(f)=\mathbf{M}_{e,y}(h)$ for every label, then $J_e(f)=J_e(h)$. If the equalities hold in every source environment, the two weighted source objectives are equal.
\end{proposition}
\begin{proof}
Write the common class moments as $\mathbf{N}_y=\mathbf{M}_{e,y}(f)=\mathbf{M}_{e,y}(h)$. Since both encoders are evaluated on the same environment, their label priors are the same. Equation~\eqref{eq:moments} therefore gives
\begin{equation}
\mathbf{M}_e(f)=\sum_y p_y^e\mathbf{N}_y=\mathbf{M}_e(h).
\end{equation}
The representation dimension and precision are fixed, so the same identity matrix and scalar $a$ are used in both objectives. The global log-determinants coincide by the displayed identity, and each conditional log-determinant coincides by the assumed equality of the corresponding class moments. Thus
\begin{align}
J_e(f)-J_e(h)
&=\frac12\left[\log\det(\mathbf{I}_d+a\mathbf{M}_e(f))
-\log\det(\mathbf{I}_d+a\mathbf{M}_e(h))\right]\notag\\
&\quad-\frac12\sum_y p_y^e\left[\log\det(\mathbf{I}_d+a\mathbf{M}_{e,y}(f))
-\log\det(\mathbf{I}_d+a\mathbf{M}_{e,y}(h))\right]=0.
\end{align}
If the assumed class-moment equalities hold separately in every source, the same argument gives $J_e(f)=J_e(h)$ for every source. Multiplication by the common source weights and summation then yield $J_{\mathrm{tr}}(f)=J_{\mathrm{tr}}(h)$. No equality of the full conditional representation distributions is needed for this conclusion.
\end{proof}
This observation is an algebraic property, not by itself an OOD impossibility theorem. Equal second moments do not specify the complete class-conditioned distributions. Conversely, the main reversal construction does not have equal \emph{class-indexed} moments across source and target: it swaps them. Its equal coding values follow from that symmetry rather than from Proposition~\ref{prop:moment}.

\subsection{Odd encoders and reflected classes}
\begin{proposition}[A restricted encoder-class blind spot]
\label{prop:odd}
Suppose binary classes satisfy $\mathbf{X}\mid Y=-1\overset{d}{=}-\mathbf{X}\mid Y=+1$. If an encoder obeys $f(-\mathbf{x})=-f(\mathbf{x})$ on the relevant support and has finite second moments, then its two conditional second moments are equal and its population coding-rate reduction is zero, for any class prior and any positive coding precision.
\end{proposition}
\begin{proof}
Let $\mathbf{H}(\mathbf{x})=f(\mathbf{x})f(\mathbf{x})^\top$. Oddness of $f$ makes this matrix-valued function even:
\begin{equation}
\mathbf{H}(-\mathbf{x})
=(-f(\mathbf{x}))(-f(\mathbf{x}))^\top
=f(\mathbf{x})f(\mathbf{x})^\top=\mathbf{H}(\mathbf{x}).
\end{equation}
Finite second moments ensure that its entries are integrable. The assumed reflection identity of the conditional input laws therefore implies
\begin{align}
\mathbf{M}_{-1}
&=\E[f(\mathbf{X})f(\mathbf{X})^\top\mid Y=-1]\notag\\
&=\E[f(-\mathbf{X})f(-\mathbf{X})^\top\mid Y=+1]\notag\\
&=\E[f(\mathbf{X})f(\mathbf{X})^\top\mid Y=+1]
=\mathbf{M}_{+1}.
\end{align}
Denote this common matrix by $\mathbf{N}$ and the two label probabilities by $p_{-1}$ and $p_{+1}$. Their sum is one, so the global moment is
\begin{equation}
\mathbf{M}=p_{-1}\mathbf{N}+p_{+1}\mathbf{N}=\mathbf{N}.
\end{equation}
For any $a>0$, $\mathbf{I}_d+a\mathbf{N}$ is positive definite. Substitution into the objective gives
\begin{equation}
J(f)=\frac12\log\det(\mathbf{I}_d+a\mathbf{N})
-\frac12(p_{-1}+p_{+1})\log\det(\mathbf{I}_d+a\mathbf{N})=0.
\end{equation}
Thus the cancellation does not require balanced classes. For nondegenerate priors both conditional laws are used; if a class has zero prior, the objective is also zero because its single positive-weight conditional moment equals the global moment.
\end{proof}
For a concrete example, let $Y$ be uniform on $\{-1,+1\}$ and $\mathbf{X}\mid Y=y\sim\mathcal N(y\boldsymbol{\mu},\sigma^2\mathbf{I}_d)$, with $\boldsymbol{\mu}\neq0$ and $\sigma>0$. The encoder $f(\mathbf{X})=\mathbf{X}/\|\mathbf{X}\|_2$ is well-defined almost surely and odd. Its coding-rate reduction is zero. Nevertheless, the rule $g(\mathbf{z})=\operatorname{sign}(\boldsymbol{\mu}^\top \mathbf{z})$ has error
\begin{equation}
P(Y\boldsymbol{\mu}^\top \mathbf{X}<0)=\Phi(-\|\boldsymbol{\mu}\|_2/\sigma),
\end{equation}
where $\Phi$ is the standard normal cumulative distribution function. To see this directly, write $\mathbf{X}=Y\boldsymbol{\mu}+\sigma\boldsymbol{\xi}$ with $\boldsymbol{\xi}\sim\mathcal N(\mathbf{0},\mathbf{I}_d)$ independent of $Y$. Multiplication by the positive scalar $1/\|\mathbf{X}\|_2$ does not change the sign of the classifier score. Moreover,
\begin{equation}
Y\boldsymbol{\mu}^\top\mathbf{X}
=\|\boldsymbol{\mu}\|_2^2+\sigma Y\boldsymbol{\mu}^\top\boldsymbol{\xi}
\sim\mathcal N(\|\boldsymbol{\mu}\|_2^2,\sigma^2\|\boldsymbol{\mu}\|_2^2).
\end{equation}
The distribution is continuous, so the zero-score event has probability zero. Standardising the score gives the displayed error formula, which approaches zero as $\|\boldsymbol{\mu}\|_2/\sigma\to\infty$. The example shows that predicting a label from a representation and distinguishing its classes through uncentred second moments are different properties. It does not show that a general nonlinear encoder is unable to learn a representation with positive coding-rate reduction. Non-odd feature maps can change these moments. This separate observation is therefore not used as a proof of the main OOD result.

\subsection{Pooling sources versus averaging their objectives}
The main text uses $\sum_e w_eJ_e(f)$. A second possible training objective is $J_{\mathrm{pool}}(f)$, computed after pooling sources with weights $w_e$. These are generally different because $\log\det$ is nonlinear. For the environmental encoder in the present model, pooling simply replaces $\delta_e$ with $\bar\delta=\sum_e w_e\delta_e$. Therefore
\begin{equation}
J_{\mathrm{pool}}(f_S)=\Jstar-\gap(\bar\delta),\qquad
J_{\mathrm{pool}}(f_C)=\Jstar.
\end{equation}
As $\bar\delta\leq\delta_{\max}$, the bound in Equation~\eqref{eq:gapbound} also bounds the pooled gap. The identical-support near-optimality construction thus applies to either objective, but their numerical values should not be identified for general encoders.

\subsection{Source probability estimates and IRM}
For $f_S$ in environment $e$, the Bayes conditional probability is
\begin{equation}
P_e(Y=1\mid f_S(\mathbf{X})=\mathbf{b}_0)=\delta_e,\qquad
P_e(Y=1\mid f_S(\mathbf{X})=\mathbf{b}_1)=1-\delta_e.
\end{equation}
When every $\delta_e<1/2$, these probabilities induce the same optimal $0$--$1$ rule. They are not equal as probability functions when the source mismatch probabilities differ. For $0<\delta_e<1$, the logistic-optimal logits are
\begin{equation}
\log\frac{\delta_e}{1-\delta_e}\quad\text{at }\mathbf{b}_0,\qquad
\log\frac{1-\delta_e}{\delta_e}\quad\text{at }\mathbf{b}_1.
\end{equation}
At $\delta_e=0$ or $1$, the corresponding logistic loss approaches its infimum only as the logits diverge; there is no finite optimal logit. Thus the environmental encoder in the distinct-source construction does not automatically satisfy a common logistic-optimal classifier condition. Agreement of decision boundaries should not be substituted for the appropriate loss-specific optimality requirement. In contrast, the coding operator $\mathbf{I}_2$ in Theorem~\ref{thm:shared} is exactly optimal in every source despite these changing conditional probabilities.

\subsection{The domain of admissible coding operators}
Equation~\eqref{eq:operator-domain} is a representation-dependent constraint. It requires unit output norm on the fixed input domain, as in the \mcr{} objective. For a general encoder image it may be more restrictive than for the two-code encoders. The proof of Theorem~\ref{thm:shared} uses only the exact identity in Equation~\eqref{eq:columns} for $f_C,f_S$ and the fact that every admissible $\mathbf{D}f$ remains unit-norm. It does not require a characterisation of $\cD(f)$ for every possible encoder.

One can instead apply a linear map and then normalise each output. For the two-code encoder, any nonzero columns again induce an arbitrary pair of unit output vectors after normalisation, giving the same inner optimisation over $(\mathbf{u},\mathbf{v})\in\sphere\times\sphere$. The stated theorem uses Equation~\eqref{eq:operator-domain} to avoid undefined normalisation at zero. It makes no statement about an unconstrained linear operator whose outputs can grow without bound.

%% file: sections/experiments.tex
\clearpage
\section{Numerical validation of the theoretical constructions}
\label{sec:numerical-validation}
We use finite-sample experiments to examine three predictions of our analysis: coding quality can survive correlation reversal, a small coding gap can coexist with severe prediction failure, and a shared optimal coding operator need not stabilise prediction.
The experiments evaluate the stable and environmental encoders defined in Equation~\eqref{eq:encoders}, and fit the classifier using source samples only. We draw
$200{,}000$ independent observations per environment and repeat each setting
with $50$ random seeds. We sample the four input-cell counts directly from their
multinomial law, which is equivalent to sampling individual observations and
then counting them. Coding values are computed from empirical second moments
and log-determinants, independently of the closed-form reference curves.
Unless otherwise stated, error bars denote one standard deviation across seeds. The source classifier minimises empirical $0$--$1$ risk over the two output codes and remains fixed during testing. The coding coefficient is $a=2/\epsilon^2$.

\paragraph{Correlation reversal preserves coding quality but changes prediction.}
We fit a classifier in a source environment with mismatch probability $0.01$
and hold it fixed while varying the target mismatch probability.
Figure~\ref{fig:numerical-reversal} shows agreement between the empirical coding
values, prediction errors, and theoretical curves. At $a=2$, the mean coding
values at target mismatch probabilities $0.01$ and $0.99$ are $0.137318$ and
$0.137250$, respectively, compared with the common population value $0.137284$.
The corresponding mean prediction errors are $0.995\%$ and $98.995\%$.
The stable encoder has zero prediction error throughout. Thus, the coding
objective can retain its value while the source label assignment becomes
unreliable. The noiseless endpoints have different input support; the
positive-noise settings share the same population input support.

\begin{figure}[htbp]
\centering
\includegraphics[width=\linewidth]{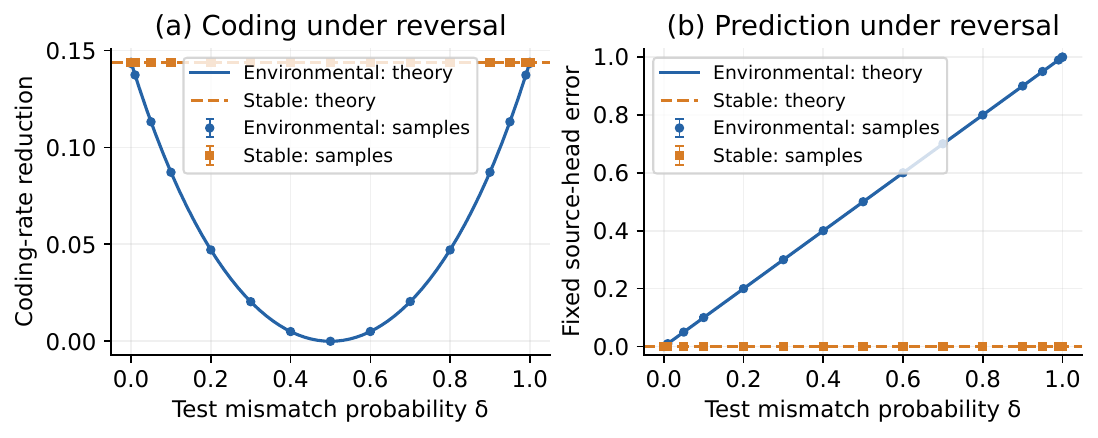}
\caption{Correlation reversal preserves coding quality while prediction fails.
Lines show theory; points show empirical means with one-standard-deviation error bars.}
\label{fig:numerical-reversal}
\end{figure}

\paragraph{A small coding gap coexists with almost maximal target error.}
We use three equally weighted source environments with mismatch probabilities
$\delta_{\max}/4$, $\delta_{\max}/2$, and $\delta_{\max}$, and a target environment
with mismatch probability $1-\delta_{\max}$. Figure~\ref{fig:numerical-gap}
compares the measured coding gap with the theoretical prediction for
$a\in\{0.5,2,8\}$. As $\delta_{\max}$ decreases, the coding gap decreases while
the target error increases. At $a=2$ and $\delta_{\max}=0.001$, the mean gap to
the population global optimum is $3.868\times10^{-4}$, compared with the
predicted $3.884\times10^{-4}$, and the mean target error is $99.900\%$.
All four possible inputs occur in every source and target sample in this
experiment; the smallest observed cell count is $15$. This rules out missing
sample support as an explanation of the observed failure in these runs.
The benchmark is the population optimum, not an asserted empirical optimum
for each randomly imbalanced sample.

\begin{figure}[htbp]
\centering
\includegraphics[width=\linewidth]{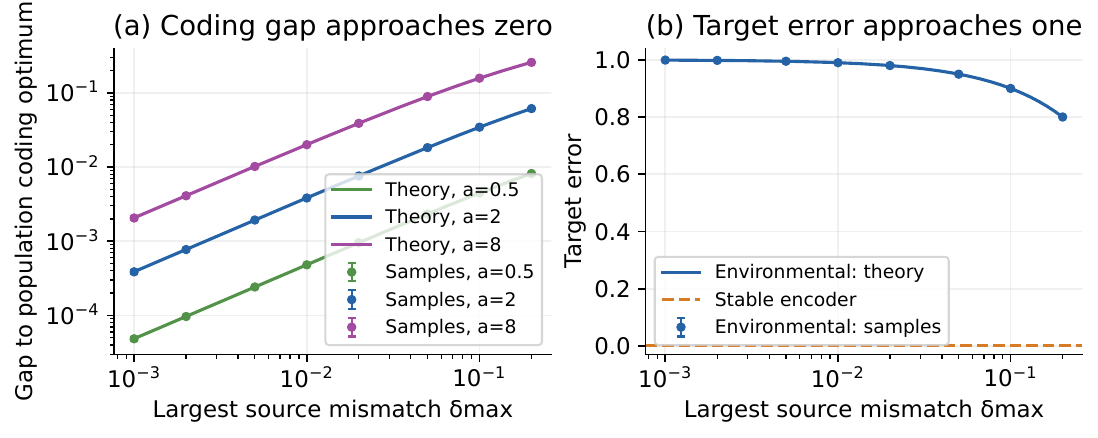}
\caption{Near-optimal coding and almost maximal target error coexist with identical input support.}
\label{fig:numerical-gap}
\end{figure}

\paragraph{A shared optimal coding operator does not stabilise prediction.}
We next use source mismatch probabilities $0.005$, $0.01$, and $0.02$, with a
reversed target at $0.99$. For each environment, we evaluate operators with two
unit columns whose relative angle ranges from $0^\circ$ to $180^\circ$ in
$2^\circ$ increments. This family includes nonorthogonal and rank-one operators,
not only rotations. In every seed and environment, the empirical grid maximum
occurs at $90^\circ$, consistent with the theorem's shared-optimality result.
Nevertheless, the fixed source classifier has approximately $99\%$ target error
at the identity operator (Figure~\ref{fig:numerical-shared}). Agreement on the
best coding operator therefore does not establish a stable label assignment.

\begin{figure}[htbp]
\centering
\includegraphics[width=\linewidth]{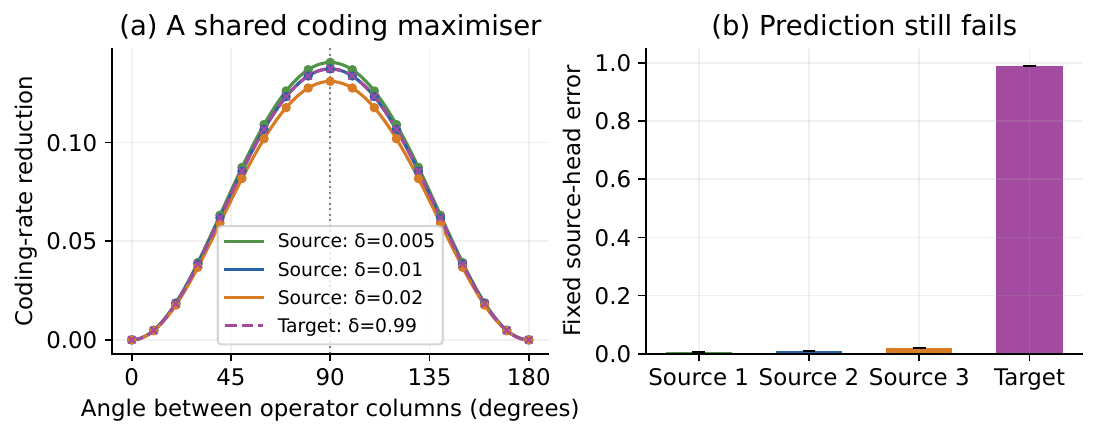}
\caption{A common coding maximiser coexists with severe target prediction failure.}
\label{fig:numerical-shared}
\end{figure}

\paragraph{More samples improve estimation, not target prediction.}
To separate estimation error from prediction failure, we fix the three source
mismatch probabilities at $0.00025$, $0.0005$, and $0.001$, the target mismatch
at $0.999$, and $a=2$. We vary the sample size per environment from $10^3$ to
$10^6$, again using $50$ independent seeds at each size.
Figure~\ref{fig:numerical-sample-size} shows the root mean squared error (RMSE)
of the empirical coding gap and target error relative to their population
predictions, computed across seeds. The coding-gap RMSE decreases from
$2.82\times10^{-4}$ to $1.10\times10^{-5}$, and the target-error RMSE decreases
from $1.05\times10^{-3}$ to $3.02\times10^{-5}$. Meanwhile, mean target error
remains between $99.874\%$ and $99.905\%$, close to the predicted $99.9\%$.
The stable encoder has zero error in every run.

Small samples can miss rare inputs even though the population support is shared.
We therefore record whether all four inputs are observed in every source and
target environment, without rejecting or resampling incomplete draws. This
occurs in $0/50$, $21/50$, $50/50$, and $50/50$ runs at the four increasing
sample sizes. At $10^6$ samples, the least frequent cell still contains at
least $103$ observations in every environment and run. Thus, in these larger
samples, severe target failure persists even when every possible input has
been observed. Increasing sample size improves agreement with the theory
without repairing the unstable predictive relationship.

\begin{figure}[!t]
\centering
\includegraphics[width=\linewidth]{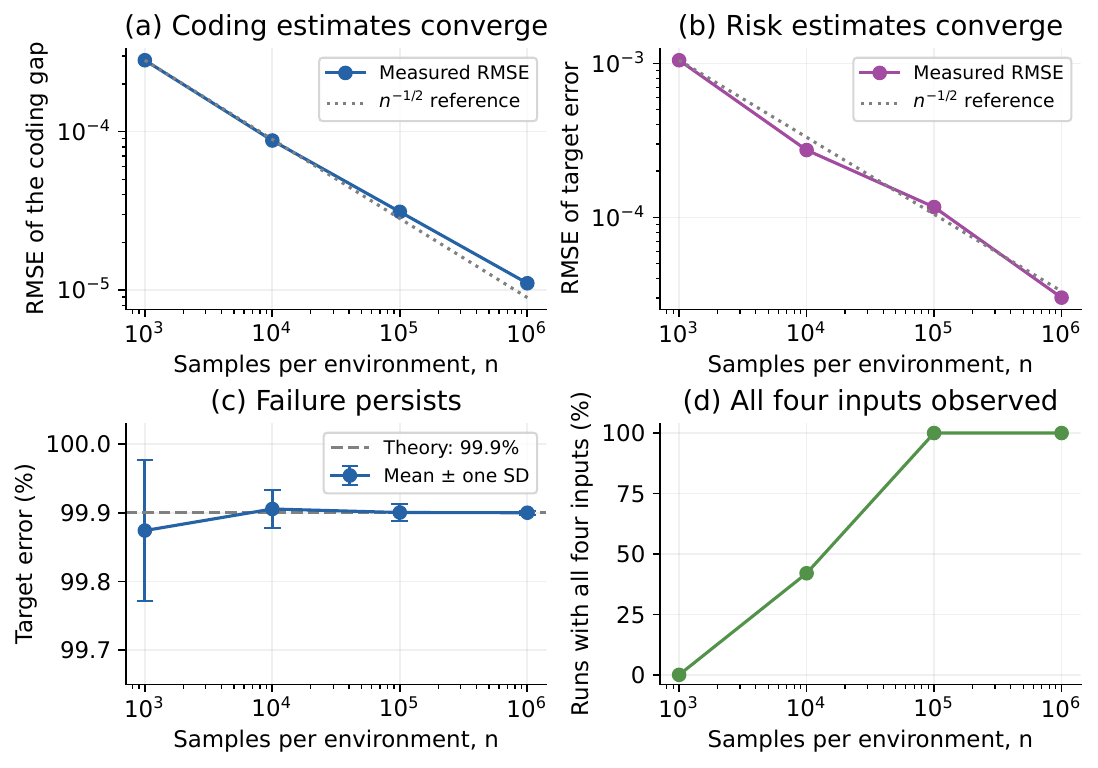}
\caption{Larger samples improve estimation while target failure persists.
RMSE is measured against the population predictions. Dotted lines in (a,b)
are $n^{-1/2}$ references, not fitted rates. Panel (c) shows means and one
standard deviation; panel (d) reports the fraction of runs observing every
input in all four environments.}
\label{fig:numerical-sample-size}
\end{figure}

\paragraph{Exact and near-optimal coding describe different cases.}
We also check the noiseless boundary separately, with source mismatch zero
and target mismatch one. Across $50$ runs of $200{,}000$ samples per
environment, the stable and environmental encoders have identical empirical
source coding values, while the environmental source classifier has zero
source error and $100\%$ target error. In this case, the source and target
supports differ. Table~\ref{tab:numerical-boundary} contrasts this boundary
with the positive-noise, identical-support construction. The exact global
optimality statement is a population result from the theory; numerical
samples need not have an exactly balanced label prior.

\begin{table}[htbp]
\centering
\small
\caption{The two failure regimes. Coding guarantees refer to the population;
errors are empirical means over $50$ runs of $200{,}000$ samples. The positive-noise row uses $\delta_{\max}=0.001$ and $a=2$.}
\label{tab:numerical-boundary}
\vspace{4pt}
\begin{tabular}{@{}llll@{}}
\toprule
Setting & Input support & Coding guarantee & Target error \\
\midrule
Noiseless reversal & Changes & Exactly optimal & $100\%$ \\
Positive noise & Identical & Near-optimal & $99.900\%$ \\
\bottomrule
\end{tabular}
\end{table}

Together, these experiments show agreement between the theoretical predictions
and finite-sample measurements. They evaluate prescribed encoders, so they do
not establish how frequently an encoder-training algorithm selects an unstable
representation. The shared-operator experiment concerns the constraint in
Equation~\eqref{eq:icr}, rather than other formulations of invariance.

%% file: sections/waterbirds_setup.tex
\section{Waterbirds experimental setup}
\label{app:waterbirds-setup}

This appendix describes the Waterbirds data pipeline, encoder training, linear-probe evaluation, and saliency computation in the supplied implementation. Hyperparameters listed below are the defaults of \texttt{Train-MCR2.py} in the anonymous supplementary code package.

\paragraph{Dataset and environment construction.}
Waterbirds~\citep{sagawa2020distributionally} provides bird images with land or water backgrounds. We use the supplied images and metadata rather than generating new composites or assigning new backgrounds. The bird label is $y\in\{0,1\}$, with $0$ denoting landbirds and $1$ denoting waterbirds; the background label is $p\in\{0,1\}$, with $0$ denoting land and $1$ denoting water. The metadata define four groups $(y,p)$. In the training split, approximately $95\%$ of each bird class has a matching background ($p=y$), and the remaining $5\%$ has a mismatched background ($p\neq y$). All four groups are pooled into a single training dataset. The encoder loss uses bird labels $y$; background labels $p$ are used to define evaluation groups. We retain the original metadata split assignments: \texttt{split=0} for training and \texttt{split=2} for testing. The implementation does not use the validation split.

The background-mismatched evaluation uses the two existing test groups $(y,p)=(0,1)$ and $(1,0)$, containing $2{,}255$ and $642$ images, respectively. This evaluates images whose background association opposes the predominant training association; it does not modify test images or create a new test split. The original test split can be evaluated in full and the two groups reported separately, or filtered to these same $2{,}897$ images before evaluation.

\paragraph{Image preprocessing.}
During encoder training, we apply a random resized crop to $224\times224$ pixels, with crop area sampled from $[0.8,1.0]$ of the original image area, followed by a random horizontal flip with probability $0.5$. For feature extraction, testing, and saliency computation, we resize the shorter image side to $256$ pixels and take a $224\times224$ centre crop. RGB values are converted to tensors and normalised using channel means $(0.485,0.456,0.406)$ and standard deviations $(0.229,0.224,0.225)$.

\paragraph{Encoder and training objective.}
The encoder is a randomly initialised ResNet-18~\citep{he2016deep}, with no pretrained weights. Its original classifier is replaced by a linear projection from the $512$-dimensional pooled backbone output to a $128$-dimensional feature vector. The projection weights use Xavier uniform initialisation and its bias is initialised to zero. Each output feature is normalised to satisfy $\|\mathbf z_i\|_2=1$. The backbone and projection are trained jointly using the negative coding-rate reduction, with equal weights on the expansion and compression terms. The compression term is computed separately for the two bird classes within each minibatch and weighted by their minibatch proportions. The coding precision is $\epsilon=1$. In the code, \texttt{eps} denotes the denominator parameter corresponding to $\epsilon^2$ in Equation~\eqref{eq:sample}; both equal one here. A diagonal term $10^{-6}\mathbf I$ is added when computing each log-determinant.

\paragraph{Optimisation settings.}
The selected script defaults to $50$ training epochs, a minibatch size of $128$, and AdamW with learning rate $3\times10^{-4}$ and weight decay $10^{-4}$. Training data are shuffled each epoch, and the final incomplete minibatch is dropped. The total gradient norm is clipped to $10$. Although a cosine learning-rate scheduler is instantiated, it is never stepped in the supplied training loop, so the effective learning rate remains constant. The PyTorch and NumPy random seeds are set to $42$, and data loading uses four workers by default. The final encoder checkpoint is saved after training.

\paragraph{Linear-probe evaluation.}
After each training epoch, the encoder is fixed and placed in evaluation mode. We extract features from all training images using the deterministic preprocessing above and fit a fresh binary logistic-regression classifier to these features and their bird labels. Its weight and bias are initialised to zero. The probe minimises the mean binary cross-entropy plus $\frac{\lambda}{2}\|\mathbf w\|_2^2$, where $\lambda=2\times10^{-3}$; the bias is not regularised. Optimisation uses L-BFGS with learning rate $1$, at most $200$ iterations, and a strong Wolfe line search. Predictions are waterbird when the logit is positive and landbird otherwise. Probe fitting does not update the encoder.

Accuracy is computed separately within each $(y,p)$ test group. For the two background-mismatched groups in Table~\ref{tab:waterbirds-test}, overall accuracy is the proportion of correct predictions among their combined $2{,}897$ images, rather than the unweighted mean of the two group accuracies. The implementation records test metrics after each epoch but does not implement checkpoint selection based on test accuracy.

\paragraph{Saliency computation.}
For a deterministically preprocessed training image $\mathbf x$, saliency is computed from the encoder outputs rather than the linear-probe logits. At each pixel $(u,v)$, the default map is
\begin{equation}
S(u,v)=\max_{c\in\{1,2,3\}}
\sum_{k=1}^{128}\left|
\frac{\partial z_k(\mathbf x)}{\partial x_{c,u,v}}
\right|.
\end{equation}
Thus, absolute input gradients are summed over all feature dimensions and then maximised over RGB channels. Map values are clipped at their $99.95$th percentile and rescaled to $[0,1]$, then overlaid on the image with opacity $0.45$. The script generates maps for the first ten training images of each class in dataset order; Figure~\ref{fig:MCR2-spurious} displays one example from each class.

\paragraph{Software and execution.}
The implementation uses PyTorch and torchvision for the model and data pipeline, NumPy for numerical utilities, Pillow for image loading, and Matplotlib for saliency visualisation. It automatically uses CUDA when available and otherwise runs on the CPU. The supplementary package specifies compatible dependency ranges rather than the exact software versions of the historical run; it does not include hardware specifications, historical run logs, or pretrained checkpoints.